\documentclass{article} 
\usepackage{iclr2027_conference,times}

\usepackage{amsmath,amsfonts,bm}

\def\eqref#1{equation~\ref{#1}}

\def\1{\bm{1}}

\DeclareMathAlphabet{\mathsfit}{\encodingdefault}{\sfdefault}{m}{sl}
\SetMathAlphabet{\mathsfit}{bold}{\encodingdefault}{\sfdefault}{bx}{n}

\usepackage{hyperref}
\usepackage{url}

\usepackage{graphicx}
\usepackage{amsmath}
\usepackage{amssymb}
\usepackage{mathtools}
\usepackage{amsthm}
\usepackage{enumitem}
\usepackage[most]{tcolorbox}
\usepackage{subcaption}
\usepackage[table]{xcolor}

\theoremstyle{plain}
\newtheorem{theorem}{Theorem}[section]
\newtheorem{proposition}[theorem]{Proposition}
\newtheorem{lemma}[theorem]{Lemma}

\theoremstyle{definition}
\newtheorem{definition}[theorem]{Definition}

\theoremstyle{remark}

\usepackage[dvipsnames]{xcolor}

\usepackage{wrapfig}
\usepackage{multirow}

\usepackage{tikz}
\newcommand{\circlednum}[1]{%
  \tikz[baseline=(char.base)]{
    \node[shape=circle, fill={rgb,255:red,33;green,95;blue,154},
          text=white, inner sep=1pt,
          font=\sffamily\bfseries\scriptsize]
          (char) {#1};
  }%
}

\newcommand{\blue}[1]{\textcolor{SkyBlue}{#1}}
\newcommand{\boldblue}[1]{\textbf{\textcolor{MidnightBlue}{#1}}}

\title{Uncertainty Localization in LLM Reasoning via Embedding Perturbations}

\author{
  Qihao~Wen \\
  University of Arizona\\
  \texttt{qihaowen@arizona.edu} \\
  \And
  Jiahao~Wang \\
  University of Arizona\\
  \texttt{jiahaow@arizona.edu}
  \And
  Yang~Nan \\
  University of Arizona\\
  \texttt{yangnan@arizona.edu}
  \And
  Pengfei~He \\
  Michigan State University \\
  \texttt{hepengf1@msu.edu}
  \And
  Ravi~Tandon \\
  University of Arizona\\
  \texttt{tandonr@arizona.edu}
  \And
  Han~Xu \\
  University of Arizona\\
  \texttt{xuhan2@arizona.edu}
}

\iclrfinalcopy 
\begin{document}

\maketitle

\begin{abstract}
Large Language Models (LLMs) have achieved significant breakthroughs across various domains, but they can still produce unreliable or misleading outputs. For responsible LLM applications, uncertainty quantification techniques are used to estimate a model’s uncertainty about its outputs, indicating the likelihood that those outputs may be problematic. For LLM reasoning tasks, it is essential to estimate uncertainty not only in the final answer but also in the intermediate reasoning process, particularly to identify where uncertainty arises. Such information may enable more fine-grained and targeted interventions during inference. In this study, we investigate which metrics can effectively localize uncertain places within an LLM reasoning trajectory. Our study reveals that uncertain intermediate continuations are more likely to occur at tokens that are highly sensitive to perturbations in the embeddings of preceding tokens. In our experiments, we show that such perturbation-based metrics achieve stronger performance in localizing uncertain intermediate steps than baseline methods, including probability-based, sampling-based, and Bayesian-based approaches. Meanwhile, our proposed metrics also enjoy good simplicity and efficiency.
\end{abstract}

\section{Introduction}

Large language models (LLMs) have demonstrated strong reasoning capabilities across a wide range of tasks~\citep{wei2022chain}, yet their reasoning processes can still contain flaws. Uncertainty quantification (UQ)~\citep{kuhn2023semantic, xiong2023can} provides a useful signal for predicting when model outputs are unreliable and require additional scrutiny. However, most existing UQ methods focus on LLM's uncertainty of the overall response or the final answer. It leaves a more fine-grained question largely unresolved: once a response is considered uncertain, \textbf{at which places in the reasoning trajectory does the uncertainty arise?} For example, an LLM can be confident in early steps but become much less certain at a specific intermediate step~\citep{zhang2026tokur, sun2026llm}, possibly making a mistake that affects the final outcome. In practice, these locations are the parts that are in more need of scrutiny. People can either manually inspect these places or invoke tools~\citep{schick2023toolformer} to improve the reasoning quality and reliability. 

To find unreliable places in an LLM reasoning response, one may consider employing external step-wise verifiers, such as process reward models (PRMs)~\citep{lightman2024let}, to assess the correctness of each step. However, such verifiers are often unavailable in many open-domain reasoning tasks~\citep{pandit2026hard2verify}. Thus, eliciting the model's inherent uncertainty remains important. On the other hand, it is potential to use the LLM's internal statistics, including token probability or entropy~\citep{li2026entropy, wang2026beyond}. They typically assume tokens with lower (generation) probability or higher entropy are less confident. Nevertheless, these metrics may not precisely reflect the model's reasoning uncertainty. For instance, a model may hesitate among multiple synonyms or different reasoning pathways. All these alternatives can be logically plausible, but may result in a relatively low probability for the selected token (see more discussion in Section~\ref{sec:pre}).

In our work, we treat the LLM reasoning process as a series of conditional probability predictions, i.e., to predict the probability for each (generated) token given preceding contents $\Pr(x_t | x_{<t})$. Our work draws a key observation: for most tokens $x_t$, if we minimally perturb their preceding tokens $x_{<t}$ (in the embedding space), the probability prediction for $x_t$ remains stable. Meanwhile, under the perturbed condition, regenerating from the position of $x_t$ can still lead to coherent and on-task continuations, indicating the small perturbation does not disrupt the reasoning process. However, there are also tokens that are more sensitive to such perturbations, e.g., the token `left' in Figure~\ref{fig:intro_case}. Regenerating from such place still yields valid continuations, but the original generation can be less likely to be reproduced under certain perturbed conditions, which reflects potential uncertainty of the model. Based on this finding, we propose two perturbation methods (random and adversarial) and measure the sensitivity of each token. We flag the highly sensitive tokens as places where uncertainty may arise. Notably, since LLMs operate in an autoregressive manner, our scores can be efficiently obtained by simultaneously perturbing embeddings of all tokens in the entire response.

In our experiments, we show the proposed metrics can effectively localize uncertain reasoning steps along an LLM trajectory. In Section~\ref{exp:incorrect}, we first consider responses containing reasoning errors and show that the positions flagged by our metrics frequently coincide with the steps where the model goes wrong. In Section~\ref{exp:correct}, we turn to the correct responses, where uncertainty can still be present. We find our method can identify ``fragile correct steps'', which are correct in the observed trajectory but can become incorrect under natural regeneration. In Section~\ref{exp:reflection}, we further demonstrate one practical application of our signals: one can identify a small subset of potentially problematic reasoning steps in the whole response, enabling more targeted inspection and intervention to improve the model generation. Together, these studies show that our method provides richer and more actionable information about reasoning uncertainty.

\begin{figure}[t]
    \centering
    \includegraphics[width=\linewidth]{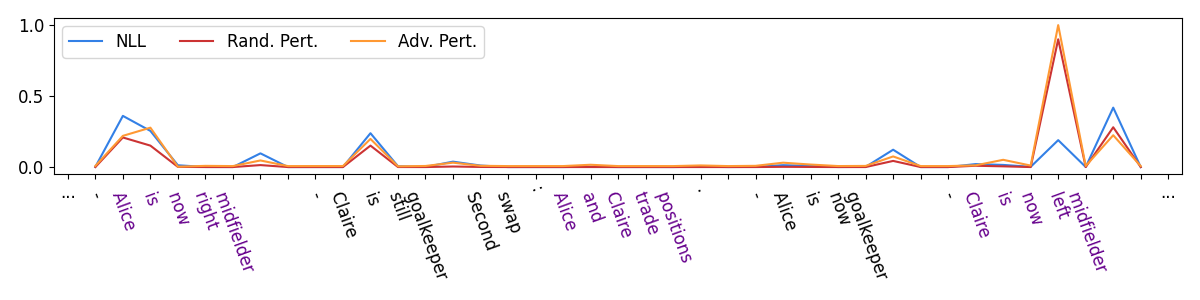}
    \caption{\small A response of Llama in logical reasoning ~\citep{suzgun2022challenging}. The model first states that \textcolor{violet}{`Alice is now right midfielder'} and \textcolor{violet}{`Alice and Claire trade positions'}, it then incorrectly infers \textcolor{violet}{`Claire is now left midfielder.'} While, the correct inference should be `Claire is now right midfielder'. Our perturbation-based metrics flag the token `left' with high uncertainty score. Note: the scores are min-max normalized in the figure.}  
    \label{fig:intro_case}
\end{figure}

\vspace{-0.2cm}
\section{Related Works}
\vspace{-0.2cm}

\subsection{Uncertainty in LLM Reasoning}\label{sec:related_work_llm}
\vspace{-0.2cm}

In LLM reasoning tasks, a major line of work studies the reasoning reliability at the response level~\citep{yadkori2024believe,kuhn2023semantic}. While, relatively fewer works study how to find uncertain regions in the generation process. Some of them focus on detecting hallucinated claims in open-form factual generation~\citep{fadeeva2024fact, sriramanan2024llm}, but the underlying mechanisms of hallucination may differ from uncertainty in LLM reasoning. 

In our considered problem in LLM reasoning, two prior studies are relevant. (1) The work~\citep{bigelow2025forking} introduces the notion of ``forking tokens,'' which are the tokens that the reasoning process can branch into different outcomes. However, this method requires repeated re-generation from individual token, which is prohibitively expensive. (2) TokUR~\citep{zhang2026tokur} adopts a Bayesian perspective and estimates token-level uncertainty by perturbing model weights to approximate Bayesian neural-network inference. In contrast, our approach perturbs token embeddings rather than model parameters, which is simpler to implement and more computationally efficient. We also conduct extensive comparisons of our method with TokUR in our experiments.

\emph{Uncertainty Localization vs.~Quantification:}
Our work focuses on identifying where uncertainty arises along an LLM reasoning trajectory. That is, given an LLM response, \textbf{our goal is to localize the steps that are relatively more uncertain}. It differs from more broadly studied uncertainty quantification problems, which is to assign a calibrated risk estimate for each prediction. Admittedly, step-wise uncertainty quantification can be challenging in reasoning tasks because different steps come from different questions and reasoning trajectories. Moreover, later steps may be confidently generated based on earlier mistakes, while still being incorrect.

\emph{Difference with Step-wise Verifiers:} 
Our study is also closely related to step-level verification models~\citep{lightman2024let}, which are external models to assess the correctness of each reasoning step. For many open-domain reasoning tasks, such verifiers are not available~\citep{pandit2026hard2verify}, which highlights the importance of the model's inherent uncertainty. Furthermore, because our scores can identify potentially problematic steps in LLM reasoning, they may also facilitate human annotation, and thereby facilitate the training of these step-wise verifier models.

\vspace{-0.1cm}
\subsection{Perturbation Methods in General Machine Learning (ML)} 
\vspace{-0.1cm}

Perturbation-based methods have been widely studied in general ML, especially in adversarial attacks and robustness~\citep{goodfellow2014explaining}. For LLM reasoning, such perturbation-based probing may be particularly useful. Unlike many traditional ML settings, a low token probability in LLM generation does not necessarily indicate model uncertainty, because it may also arise from multiple plausible continuations. Therefore, perturbation-based methods provide a complementary way to capture prediction stability beyond probability-based scores. Indeed, TokUR~\citep{zhang2026tokur}, which perturbs model weights, is a relatively strong baseline in our experiments, while our methods further improve uncertainty localization. These results together suggest that perturbation may provide a particularly informative signal for reasoning uncertainty.

\vspace{-0.2cm}
\section{Methodology}\label{sec:method}
\vspace{-0.2cm}

In this section, we first present the necessary definitions and notations. Then, we present details of our proposed perturbation-based method to localize uncertain places in LLM reasoning. 

\vspace{-0.2cm}
\subsection{Preliminary and Definition}\label{sec:pre}
\vspace{-0.2cm}

We consider a given LLM and query it to perform a reasoning task. For an input query $\mathbf{q} = (x_1, ..., x_m)$, e.g., a math reasoning question with $m$ tokens, the LLM generates a response with $n$ tokens and we denote it as $\mathbf{r} = (x_{m+1}, x_{m+2}, ...x_{m+n})$. For some metrics introduced later, we will extract the token embeddings (before Transformer layers) of all tokens including both input tokens and output tokens. It will result in $(m+n)$ embedding vectors defined as $\mathbf{H} = (\mathbf{h}_1, ..., \mathbf{h}_{m+n})$, where each $\mathbf{h}_i \in \mathbb{R}^d$ and $d$ is the dimension of token embedding space. Based on them, we give the definition of LLM's computation on the (generation) probability of each next token as below.

\begin{definition}[\textbf{Token (Generation) Probability}]
Given a sequence of tokens $(x_1, ..., x_{t-1})$ with $t>m$, an auto-regressive LLM calculates the probability of each potential next token $v$ in vocabulary $\mathcal{V}$ as $
\Pr(v \mid x_1, x_2, \ldots, x_{t-1})\text{~~for~~} v\in \mathcal{V}.$
For the actually sampled token $x_t$ in the response, we denote the token probability of $x_t$ as:
\begin{equation}\label{eq:prob}
    \mathrm{P}(x_t) = \Pr(x_t \mid x_1, x_2, \ldots, x_{t-1})
\end{equation}
\end{definition}
It is also frequently formed in Negative Log Likelihood (NLL) as: $\mathrm{NLL}(x_t) = -\log \mathrm{P}(x_t)$. Intuitively, token probability indicates how strongly the LLM believes the next token is $x_t$, given the preceding tokens $x_1,..., x_{t-1}$.
Besides, one related metric is token entropy, which measures the dispersion of the next-token's probability distribution over the vocabulary $\mathcal{V}$.
\begin{definition}[\textbf{Token Entropy}]
The entropy when generating token $x_t~(m < t \leq m+n)$ is:
\begin{align}
\mathrm{Entropy}(x_t)
= - \sum_{v \in \mathcal{V}}
\Big(\Pr(v \mid x_1, x_2, \ldots, x_{t-1}) \cdot
\log \Pr(v \mid x_1, x_2, \ldots, x_{t-1})\Big)
\label{eq:entropy}
\end{align}
\end{definition}
In general, token probability and entropy may reflect a model's confidence when generating each token \(x_t\)~\citep{tian2023just}. However, they are insufficient to reflect the model's reasoning uncertainty, as they can be confounded by ``benign variability''~\citep{kuhn2023semantic}.  Figure~\ref{fig:legitimate} shows an illustration in logical reasoning from the BIG-Bench Hard dataset~\citep{suzgun2022challenging}. In Part~(A), at token position \circlednum{1}, the model generates the token ``At'' with a low probability of \(0.252\). However, this low probability may simply arise because other plausible alternatives are also available, such as ``After'' and ``Now'' with probability 0.222 and 0.173 respectively. Refer to Part (A), if we regenerate from this position, it can  lead to these alternative choices for continuation but conveying plausible and even similar meaning. This suggests the model's reasoning uncertainty is not likely to arise at the positions like \circlednum{1}.

\begin{figure}[t]
    \centering
\includegraphics[width=0.98\linewidth]{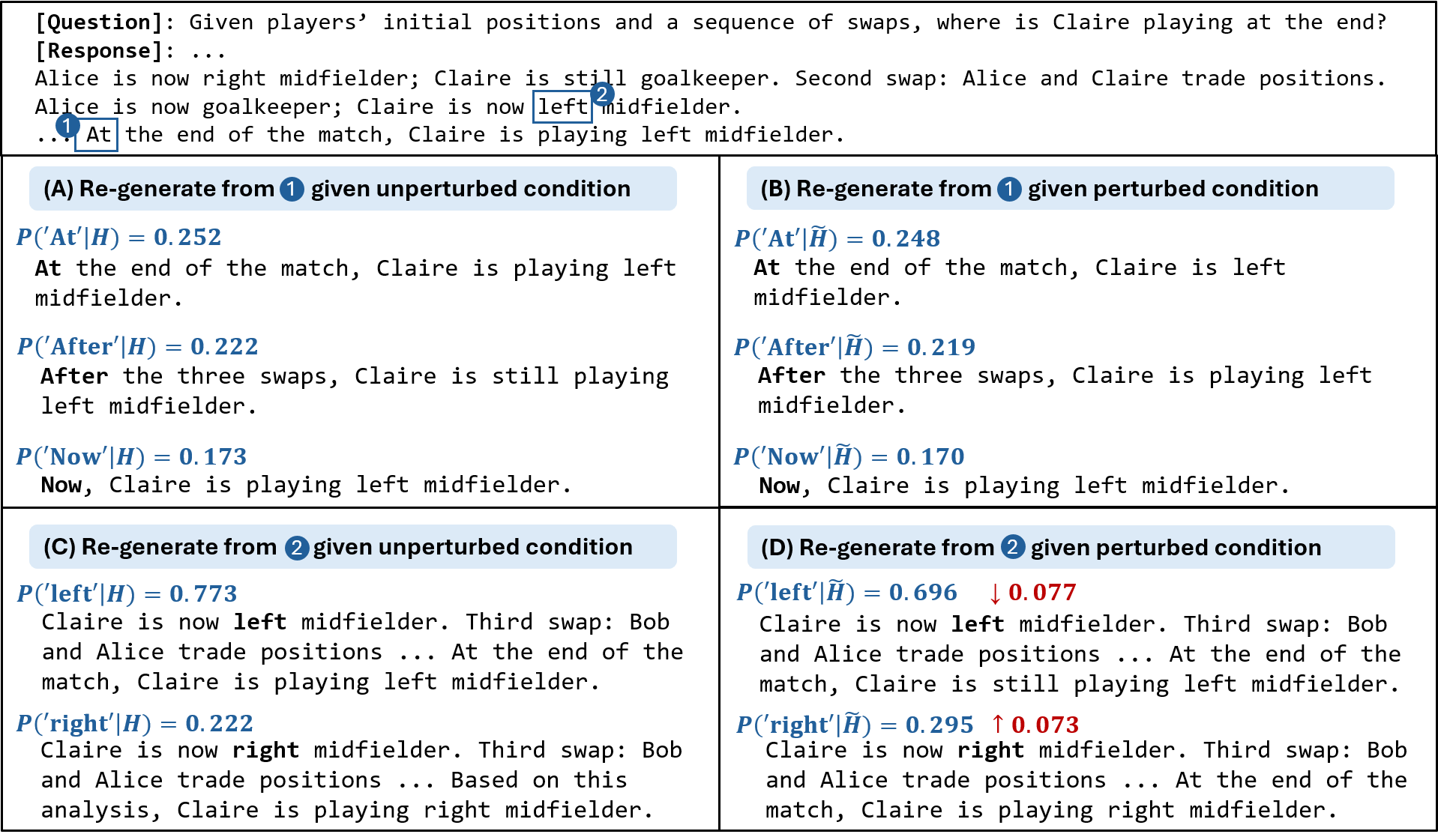}
\vspace{-0.1cm}
    \caption{The model's alternative tokens at key positions, and the regenerated continuations for the example in Figure~\ref{fig:intro_case}.  We provide additional positions \& examples in Figure~\ref{fig:legitimate_extra} in Appendix~\ref{sec:appendix_cases}.}
    \label{fig:legitimate} 
\end{figure}

\vspace{-0.3cm}
\subsection{Random Perturbation}\label{sec:pert}
\vspace{-0.2cm}

To alleviate the aforementioned challenge, we treat the LLM token log-probability (Eq.\ref{eq:prob}) as a function of previous token embeddings $\mathbf{H}_{1:t-1}$, and we define: 
\begin{equation}
f_t(\mathbf H_{1:t-1}) = \log \Pr \left(x_t \mid \mathbf H_{1:t-1}\right).
\end{equation}
Then, we measure this function’s
sensitivity to minor perturbations on $\mathbf{H}_{1:t-1}$. In detail, given a complete response, we apply i.i.d Gaussian noises 
$\boldsymbol{\varepsilon} = (\boldsymbol{\epsilon}_1, \boldsymbol{\epsilon}_2,...,\boldsymbol{\epsilon}_{m+n})$, with distribution $\mathcal N(0, \sigma^2 I)$ and dimension $(m+n)\times d$, onto all token embeddings including both the query $\mathbf{q}$ and the whole response $\mathbf{r}$. We define the perturbed embeddings as $\tilde{\mathbf{H}} = \mathbf{H} + \boldsymbol{\varepsilon}$. For each generated token $x_t$, we measure the variance of token probability under $l$ times of perturbation on $\mathbf{H}$:
\begin{equation}
\mathrm{Pert.}(x_t) = 
\mathrm{Var}_{\boldsymbol{\varepsilon}}
\!\left[\log 
\Pr(x_t \mid \tilde{\mathbf{H}}_{1:t-1})
\right].
\label{eq:rand_perturb}
\end{equation}
Since LLMs are autoregressive, the probability for each generated token $x_t$ only depends on its preceding contents. It means we only need to sample $\boldsymbol{\varepsilon}$ to perturb the entire sequence once, instead of
performing separate perturbations for each token. In Figure~\ref{fig:legitimate}, we show two representative positions, \circlednum{1} and \circlednum{2} with low and high sensitivity scores respectively, together with their regenerated continuations. This illustration reveals two key properties as introduced below.

\textbf{Embedding perturbation does not disrupt reasoning coherence.} 
In Figure~\ref{fig:legitimate}, for a considered position such as \(\circlednum{1}\) with $x_t$=``At'', we regenerate continuations from the perturbed prefix \(\tilde{\mathbf{H}}_{1:t-1}\). We find that the resulting continuations remain coherent and on-task. This is true even when regenerating from position \(\circlednum{2}\), which has a high sensitivity score. Thus, for the original \(x_t\) and its subsequent continuation, we can treat the perturbed prefix \(\tilde{\mathbf{H}}_{1:t-1}\) as a valid neighboring condition. This is also why we perturb in the embedding space, as it may have a relatively limited impact on the validity of the resulting conditions~\citep{zhu2019freelb}.  In Figure~\ref{fig:legitimate_extra} in Appendix~\ref{sec:appendix_cases}, we provide more regeneration outcomes at different positions.

\textbf{{Less-Supportive Neighboring Condition.}} Given these valid neighboring conditions, most tokens in the response have low \(\mathrm{Pert.}(x_t)\), including token ``At'' at position \(\circlednum{1}\). It suggests that they remain stably supported across these perturbed conditions. Differently, if \(\mathrm{Pert.}(x_t)\) is high (token ``left'' at position \(\circlednum{2}\)), there can exist a  \(\tilde{\mathbf H}_{1:t-1}\) under which the probability of \(x_t\) drops by a non-negligible margin, as shown in Part~(D). It means the original generation of $x_t$ is not stably supported by neighboring conditions, and may have a lower probability of being sampled under certain \(\tilde{\mathbf H}_{1:t-1}\). 
Formally, we refer to \(\tilde{\mathbf H}_{1:t-1}\) which reduces the probability of \(x_t\) as a ``less-supportive neighboring condition'' for \(x_t\). The following proposition shows that a higher \(\mathrm{Pert.}(x_t)\) implies a higher probability of encountering a less-supportive neighboring condition.

\begin{proposition}
We define $\tilde{\mathbf H}=\mathbf H+\boldsymbol{\varepsilon}$, where
$\boldsymbol{\varepsilon}\sim\mathcal N(0,\sigma^2 I)$.
Under first-order approximation and a sufficiently small $\sigma^2$, for any $c>0$, we have:
\[
\Pr_{\boldsymbol{\varepsilon}}
\left[
f_t(\tilde{\mathbf H}_{1:t-1})
\le
f_t(\mathbf H_{1:t-1})-c
\right]
\approx
\Phi\left(
-\frac{c}{\sqrt{\mathrm{Pert.}(x_t)}}
\right).
\]
A higher $\mathrm{Pert.}(x_t)$ indicates a higher probability that $f_t(\tilde{\mathbf H}_{1:t-1}) - f_t({\mathbf H}_{1:t-1}) \leq -c$, implying such $\tilde{\mathbf H}_{1:t-1}$ is a less-supportive neighboring condition for $x_t$. The proof is in Appendix~\ref{app:proof}.
\label{prop:less_supportive_neighboring_condition}
\end{proposition}

Besides, another interesting observation is that high sensitivity to perturbations indicates that the prediction of \(x_t\) can be strongly influenced by the preceding context. In reasoning trajectories, such tokens may be more likely to be reasoning-relevant tokens, carrying information derived from earlier reasoning steps rather than serving as function words such as ``At'' at \(\circlednum{1}\).

\vspace{-0.2cm}
\subsection{Adversarial Perturbation}\label{sec:pert_based_method}
\vspace{-0.1cm}

We also consider adversarial perturbation strategy~\citep{goodfellow2014explaining}. In detail, we first calculate the response-level probability, as the sum of log probabilities of all generated tokens:
\begin{equation}
\mathcal{F}(\mathbf H)
=
\sum_{i=m+1}^{m+n}
\log \Pr
\left(
x_i \mid \mathbf H_{1:i-1}
\right).
\end{equation}
We perturb all embeddings $\mathbf{H}$ along the direction that decreases the total probability:
$\hat{\mathbf H}=\mathbf H-
\alpha\operatorname{sgn} (\nabla_{\mathbf H}\mathcal{F}(\mathbf H))$, where $\alpha$ is a small constant. Here, we choose the sign operation to keep the perturbation on each token embedding the same magnitude. Then, we aim to find the generated tokens that become less supported under such a perturbation: 
\begin{equation}\label{eq:adv}
\mathrm{Pert.}(x_t)
=
\log \Pr(x_t \mid \mathbf H_{1:t-1})
-
\log \Pr(x_t \mid \hat{\mathbf H}_{1:t-1}).
\end{equation}

Under this formulation, only one backward pass and one forward pass are needed to obtain the scores for all positions. Compared to random perturbation, this metric directly explores a less-supportive neighboring condition $\hat{\mathbf H}_{1:t-1}$ that reduces the probability of $x_t$. Moreover, the following proposition shows that \(\hat{\mathbf H}_{1:t-1}\) also less supports the overall response.
 
\begin{proposition} 
For a small constant $\alpha$, a first-order Taylor expansion of Eq.\ref{eq:adv} gives:
\begin{equation}
\mathrm{Pert.}(x_t)
\approx \alpha \cdot \langle \nabla_{\mathbf{H}_{1:t-1}} \log \Pr(x_t \mid \mathbf{H}_{1:t-1}) ~,~ \operatorname{sgn} (\nabla_{\mathbf H_{1:t-1}}\mathcal{F}(\mathbf H)) \rangle\\
\label{eq:adv_directional}
\end{equation}
\label{prop:adv_directional}
\end{proposition}
We can see that a high sensitivity score indicates strong alignment between the gradients of the token-level and response-level probabilities, both with respect to \(\mathbf{H}_{1:t-1}\). In other words, a direction that decreases \(\Pr(x_t \mid \mathbf{H}_{1:t-1})\) also tends to decrease the probability of the overall reasoning trajectory. This suggests that the identified highly sensitive tokens are not only locally unstable, but their instability is also closely related to the overall reasoning process. Similar to Figure~\ref{fig:legitimate}, we provide the regeneration outcomes in Figure~\ref{fig:legitimate_adv} in Appendix for adversarial perturbation.

\vspace{-0.3cm}
\section{Experiment}\label{sec:main_exp}
\vspace{-0.3cm}

In our experiment, \textbf{we first evaluate the localization ability of our method in scenarios that the reasoning uncertainty exists. Then, we discuss how it can assist practical LLM reasoning.} In detail, we first assess:
{Given that a model contains unreliable steps probably due to uncertainty, whether our method can indicate the location of them?}\footnotemark Answering this question might not be straightforward as the model uncertainty itself is hidden. Thus, we rely on observable reasoning risks as proxy ground-truth for uncertain steps. We consider two types of \emph{proxy uncertain steps} and we examine whether our method can successfully locate them in the whole reasoning trajectory:

\footnotetext{We consider natural reasoning responses rather than procedural ones with predefined step boundaries. We segment each response into sentences and treat each sentence as an individual reasoning step. More details is in Appendix~\ref{sec:exp_detail}}

\begin{itemize} 
\item \textbf{First Wrong Step}: Many existing studies assume if the model is uncertain, it is likely to make errors~\citep{jiang2021can,kadavath2022language}. Following prior works, we also use incorrectness as the proxy ground-truth for uncertainty. Notably, a reasoning trajectory can have multiple wrong steps, and later wrong steps might be confidently derived from earlier mistakes. Therefore, we only focus on the ``first wrong step'' as \textit{proxy uncertain step}, which are the places the model first goes wrong from correct contents (Section~\ref{exp:incorrect}).

\item \textbf{Fragile Correct Step}: While a reasoning step is correct, uncertainty can still exist~\citep{bigelow2025forking}. As a complement to Section~\ref{exp:incorrect}, we therefore define the second type of \textit{proxy uncertain steps} as: the correct steps in the realized trajectory which are prone to be incorrect upon (natural) re-generation. (Section~\ref{exp:correct}).
\end{itemize}

In practice, one can first use response-level UQ to assess the overall reliability of a response~\citep{kuhn2023semantic}. Once a potentially unreliable response is identified, our method can then localize the uncertain positions within its reasoning trajectory. In Section~\ref{exp:reflection}, we further provide an example application that uses this capability to practically assist LLM reasoning under this pipeline.

\begin{table}[t]
    \centering
    \caption{Performance of each method in 300 responses with incorrect steps, The 95\% conf.~interval is $\pm 0.057$ for 300 Bernoulli trials with probability around 0.5.}
    \label{tab:incorrect}
    \vspace{-0.2cm}\resizebox{0.96\linewidth}{!}{

    \begin{tabular}{c|cc|cc|cc|cc|cc|cc}
    \hline 
    &     \multicolumn{6}{c|}{\textbf{Math}} &  \multicolumn{6}{c}{\textbf{Logic}}
    \\
    & \multicolumn{2}{c}{\textit{Llama}}
    & \multicolumn{2}{c}{\textit{Qwen}}
    & \multicolumn{2}{c|}{\textit{Mistral}}
    & \multicolumn{2}{c}{\textit{Llama}}
    & \multicolumn{2}{c}{\textit{Qwen}}
    & \multicolumn{2}{c}{\textit{Mistral}}\\
    \hline
    &  10\%  & 20\%  &  10\%  & 20\%  &  10\%  & 20\%
    &  10\%  & 20\%  &  10\%  & 20\%  &  10\%  & 20\%
    \\
    NLL-max  &  0.19 & 0.30  &  0.21 & 0.38  &  0.16 & 0.28  &  0.21 & 0.41  &  0.27 & 0.50  &  0.15 & 0.32 \\
    Entropy-max  &  0.20 & 0.34  &  0.23 & 0.41  &  0.14 & 0.24  &  0.18 & 0.39  &  0.25 & 0.50  &  0.13 & 0.30 \\
    TokUR-max  &  0.27 & 0.48  &  0.26 & 0.45  &  0.28 & 0.44  &  0.27 & 0.42  &  0.29 & 0.47  &  0.25 & 0.45 \\
    \arrayrulecolor{gray!40}
    \hline
    \arrayrulecolor{black}
    NLL-avg  &  0.15 & 0.25  &  0.16 & 0.32  &  0.10 & 0.19  &  0.18 & 0.39  &  0.20 & 0.39  &  0.10 & 0.26 \\
    Entropy-avg  &  0.15 & 0.27  &  0.15 & 0.35  &  0.10 & 0.19  &  0.19 & 0.41  &  0.19 & 0.39  &  0.12 & 0.27 \\
    TokUR-avg  &  0.21 & 0.35  &  0.19 & 0.35  &  0.23 & 0.42  &  0.22 & 0.40  &  0.23 & 0.42  &  0.24 & 0.39 \\
    \arrayrulecolor{gray!40}
    \hline
    \arrayrulecolor{black}
    CCP  &  0.20 & 0.35  &  0.31 & \blue{0.50}  &  0.21 & 0.34  &  0.27 & 0.44  &  0.30 & 0.54  &  0.20 & 0.40 \\
    TokenSAR  &  0.20 & 0.29  &  0.26 & 0.42  &  0.19 & 0.31  &  0.29 & 0.47  &  0.30 & 0.54  &  0.22 & 0.41 \\
    Graph Uncertain.  &  0.21 & 0.28  &  0.20 & 0.32  &  0.15 & 0.21  &  0.23 & 0.38  &  0.31 & 0.44  &  0.16 & 0.30 \\
    \arrayrulecolor{gray!40}
    \hline
    \arrayrulecolor{black}
    Rand.~Pert.  &  \blue{0.40} & \blue{0.54}  &  \blue{0.33} & \blue{0.50}  &  \blue{0.31} & \boldblue{0.50}  &  \blue{0.33} & \blue{0.50}  &  \blue{0.38} & \blue{0.55}  &  \blue{0.28} & \blue{0.46} \\
    Adv.~Pert.  &  \boldblue{0.45} & \boldblue{0.59}  &  \boldblue{0.35} & \boldblue{0.52}  &  \boldblue{0.34} & \boldblue{0.50}  &  \boldblue{0.42} & \boldblue{0.58}  &  \boldblue{0.40} & \boldblue{0.57}  &  \boldblue{0.35} & \boldblue{0.50} \\
    \arrayrulecolor{black}
    \hline 
    \end{tabular}
    }
    
\end{table}

\vspace{-0.2cm}
\subsection{Localizing Uncertain Steps: First Wrong Steps}\label{exp:incorrect}
\vspace{-0.2cm}

\textbf{Basic Setups.}
We first focus on basic math reasoning problems in MATH dataset~\citep{hendrycks2021measuring}, and pure-logical reasoning in BIG-Bench~Hard~(BBH)~\citep{suzgun2022challenging}, and investigate open-source LLMs: Llama-3.1-8B-Instruct~\citep{dubey2024llama}, Qwen-2.5-7B-Instruct~\citep{qwen2} and Mistral-7B-Instruct-v0.3 \citep{jiang2023mistral7b}. In addition, we also consider more challenging dataset GPQA~\citep{rein2024gpqa} with more advanced models. Since the
experiment is focused on reasoning tasks, we set a low sampling temperature of all LLMs to
$T = 0.2$. In Appendix~\ref{app:extra_result}, we report the results under
$T = 0.5$.  For hyperparameters, we use $l=20$ random perturbation with $\sigma = 0.001$, and adversarial perturbation with scale $\alpha = 0.0001$. We analyze hyperparameters and various system prompts, and show our method is robust in Appendix~\ref{app:extra_result}.

\textbf{Evaluation Pipeline \& Ground-truth Annotation.} For each dataset, we collect responses containing reasoning errors\footnotemark, and we point out the ``first wrong step'' in each of them. This annotation is conducted by GPT-5.6, and we conduct a comprehensive manual examination. Especially, for the task of finding the ``first wrong step'', we inspect MATH and BBH for 180 responses in total, and observe a high human alignment rate 93.9\%. The details are in Appendix~\ref{sec:prompts}.

Then, we evaluate whether a metric can successfully identify this ``first wrong step''. In detail, for a trajectory with multiple steps, we check the top 10\% or top 20\% uncertain steps flagged by each method, and report whether the annotated ``first wrong step'' is among them. Notably, our method and some baselines provide uncertainty scores at the token-level rather than the step-level. For our method, we aggregate token-level scores within each step by taking the maximum. For other token-level baselines, we consider both maximum and mean aggregation.

\footnotetext{Responses with incorrect steps can still possibly yield a correct final answer although uncommon. Table~\ref{tab:incorrect} includes the results for all responses containing incorrect reasoning steps, regardless of final-answer correctness. We also provide Table~\ref{tab:incorrect_final} which includes only responses with incorrect final answers.}

\textbf{Baselines.}
We consider a range of baselines including: Negative Log-Likelihood (NLL), Token Entropy, and TokUR~\citep{zhang2026tokur} (TokUR conducts weight perturbation). These three methods provide token-level signals, which we aggregate into step-level scores by taking the maximum and average over the tokens within each step. We also have step-level baselines: CCP~\citep{fadeeva2024fact}, TokenSAR~\citep{duan2023shifting} which are originally for hallucination detection in open-form generations; and Graph Uncertainty~\citep{jiang2024graph} which conducts multi-sampling and assess whether each sentence is consistently produced in different generations.

\textbf{Result.} Table~\ref{tab:incorrect} shows results on 300 responses with incorrect reasoning steps from each of MATH and BBH. We report the fraction of responses in which the first wrong step is covered by the top 10\% or 20\% most uncertain steps. From the result, we can see our perturbation-based methods have clear and consistent advantage over all baselines, and Adv.~Pert.~demonstrates stronger performance than Rand.~Pert in most cases. For baselines, we notice that TokUR~\citep{zhang2026tokur} yields generally better performance than other token-level baselines, suggesting the perturbation-based strategies may be generally effective for this task. Compared with TokUR, our methods are simpler to implement and achieve stronger localization performance.
Besides, the efficacy of Graph Uncertainty~\citep{jiang2024graph} (multiple-sampling) is usually low, as it struggles to pinpoint the origin of errors (it will also flag the later steps after the first wrong sentence).

In  Appendix~\ref{sec:appendix_cases}, we provide additional visualizations of our scores on MATH and BBH, to demonstrate our scores can identify diverse types of reasoning errors, including problem misinterpretation, calculation errors, and logical inference errors. Moreover, these cases further support that our scores indeed highlights ``reasoning-relevant'' tokens, as mentioned in Section~\ref{sec:pert}.

\begin{wraptable}{r}{0.44\textwidth}
\centering
\vspace{-0.3cm}
\caption{Performance on GPQA.}
\vspace{-0.2cm}
\label{tab:gpqa}
\resizebox{\linewidth}{!}{
    \begin{tabular}{c|cc|cc}
    \hline 
        & \multicolumn{4}{c}{\textbf{GPQA}} \\
    & \multicolumn{2}{c}{\textit{gpt-oss-20b}}
    & \multicolumn{2}{c}{\textit{Qwen3.5-27B}} \\
    \hline
    &  10\% & 20\%&  10\% & 20\% \\
    NLL-max  &  0.21 & 0.30  &  0.23 & 0.38 \\
    Entropy-max  &  0.18 & 0.28  &  0.28 & 0.39 \\
    \arrayrulecolor{gray!40}
    \hline
    \arrayrulecolor{black}
    NLL-avg  & 0.13 & 0.31  &  0.13 & 0.28 \\
    Entropy-avg  &  0.15 & 0.28  &  0.15 & 0.30 \\
    \arrayrulecolor{gray!40}
    \hline
    \arrayrulecolor{black}
    CCP  &  0.23 & 0.35  &  \blue{0.30} & 0.46 \\
    TokenSAR  &  0.18 & 0.30  &  0.27 & \blue{0.49} \\
    \arrayrulecolor{gray!40}
    \hline
    \arrayrulecolor{black}
    Rand. Pert.  &  \blue{0.26} & \blue{0.42}  &  0.29 & 0.48 \\
    Adv. Pert.  &  \boldblue{0.32} & \boldblue{0.43}  &  \boldblue{0.36} & \boldblue{0.50} \\
    \hline 
    \end{tabular}
    }
\end{wraptable}

Table~\ref{tab:gpqa} provides an analysis on a more challenging GPQA benchmark~\citep{rein2024gpqa}, which requires more advanced reasoning capability. Here, we consider stronger models including gpt-oss-20b~\citep{agarwal2025gpt} and Qwen3.5~\citep{yang2025qwen3} to handle this dataset. These models also contain fewer simple mistakes which can be made by Llama~3.1 and Qwen~2.5. Due to the substantially longer responses ($\approx$2,000 tokens in each response), we collect 100 responses with incorrect reasoning steps for each model and exclude Graph Uncertainty because of its high computational cost. We also exclude TokUR because adapting its low-rank weight perturbation to gpt-oss and Qwen3.5 requires non-trivial, model-specific implementation. As shown in Table~\ref{tab:gpqa}, our methods remain effective on these more challenging problems, demonstrating their broad applicability across different settings. Interestingly, CCP~\citep{fadeeva2024fact} and TokenSAR~\citep{duan2023shifting} perform more strongly on GPQA than on MATH and BBH. One possible explanation is that these methods are designed for factual hallucination detection, while factual errors may occur more frequently in harder knowledge-intensive questions such as GPQA. In Section~\ref{sec:hallucination}, we further investigate the potentially different mechanisms in factual hallucination and reasoning uncertainty.

\vspace{-0.2cm}
\subsection{Localizing Uncertain Steps: Fragile Correct Steps}\label{exp:correct}
\vspace{-0.2cm}

In LLM reasoning, a model may produce a correct intermediate step even when uncertainty is still present. This is also suggested by self-consistency methods~\citep{wang2022self}: generations from the same context can lead to both correct and incorrect final outcomes. This implies that some correct steps in a single trajectory may not be stably reproduced in other generations. We therefore study whether our method can identify correct-but-uncertain steps, which we refer to as \emph{fragile correct steps}. Formally, for a correct step \(s_i\), we define the ``regeneration risk'' to denote whether naturally re-generating from its preceding steps, $s_{1:i-1}$, can cause an error at the position of $s_i$:
\begin{equation}\label{eq:risk}
    \text{risk}(s_i) = \Pr\left(\text{incorrect } s^*_i \mid \text{correct } s_{1:i-1} \right)
\end{equation}
That is, we re-generate one sentence from \(s_{1:i-1}\) multiple times (\(n=10\)) and use GPT-5.6 to examine the outputs, and estimate the probability an error emerges immediately at step \(s_i\).
In the experiment, we focus on responses whose intermediate steps are all correct but that still exhibit overall uncertainty. We calculate such overall uncertainty via sampling multiple responses for the same question and check if self-consistency \(\leq 0.7\). We then evaluate whether our method can identify fragile steps with high risk scores in Eq.~\ref{eq:risk}. To reduce re-generation cost, we compute the regeneration risk only for the top~10\% and top~20\% most uncertain sentences identified by each method, and report the maximum risk at each cutoff.

\textbf{Result.}  Table~\ref{tab:correct} reports the average results over 100 collected responses from MATH and BBH. Because it need huge number of resampling and additional judgments, we only keep the representative token-level baselines for comparison. From the result, we observe that Adv.~Pert.~and Rand.~Pert.~also consistently outperforms the baselines, especially given a smaller range (top 10\%). This may suggest that such fragile correct steps are relatively sparse. Overall, this result shows that our methods capture uncertainty even when the realized reasoning steps are correct, rather than merely serving as incorrectness indicators.

\begin{table}[t]
\centering
\caption{The max re-generation risk of flagged uncertain steps (95\% conf.~interval is $\pm 0.049$)}
\vspace{-0.3cm}
\label{tab:correct}
\resizebox{0.68\linewidth}{!}{
    \begin{tabular}{c|cc|cc|cc|cc}
    \hline 
        & \multicolumn{4}{c|}{\textbf{Math}}        & \multicolumn{4}{c}{\textbf{Logic}}\\
    & \multicolumn{2}{c}{\textit{Llama}}
    & \multicolumn{2}{c|}{\textit{Qwen}}
    & \multicolumn{2}{c}{\textit{Llama}} 
        & \multicolumn{2}{c}{\textit{Qwen}} \\
    \hline
    &  10\% & 20\%&  10\% & 20\% &  10\% & 20\%&  10\% & 20\%  \\
    NLL-max  &  0.22 & 0.29  &  0.14 & 0.20  &  0.20 & 0.30  &  0.10 & \blue{0.19} \\
    Entropy-max  &  0.23 & 0.30  &  0.14 & 0.19  &  0.22 & 0.30  &  0.09 & 0.18 \\
    TokUR-max  &  0.23 & \blue{0.31}  &  0.14 & 0.23  &  0.23 & \blue{0.33}  &  0.13 & 0.17 \\
    \arrayrulecolor{gray!40}
    \hline
    \arrayrulecolor{black}
    Rand. Pert.  &  \blue{0.27} & \blue{0.31}  &  \blue{0.21} & \blue{0.27}  &  \blue{0.24} & 0.31  &  \blue{0.17} & \blue{0.19} \\
    Adv. Pert.  &  \boldblue{0.31} & \boldblue{0.35}  &  \boldblue{0.22} & \boldblue{0.29}  &  \boldblue{0.31} & \boldblue{0.35}  &  \boldblue{0.19} & \boldblue{0.22} \\
    \hline 
    \end{tabular}
    }

\vspace{0.4cm}

\centering
\caption{The accuracy with external revision on uncertain steps.}
\vspace{-0.3cm}
\label{tab:revision}
\resizebox{0.9\linewidth}{!}{
    \begin{tabular}{c|cc|cc|cc|cc|cc|cc}
    \hline 
        & \multicolumn{4}{c|}{\textbf{Math (Precalculus)}}  &  \multicolumn{4}{c|}{\textbf{Math (Cnt. \& Prob.)}}  &  \multicolumn{4}{c}{\textbf{Logic (Tracking Shuffled)}}\\
    & \multicolumn{2}{c}{\textit{Llama}}
    & \multicolumn{2}{c|}{\textit{Qwen}}
    & \multicolumn{2}{c}{\textit{Llama}}
    & \multicolumn{2}{c|}{\textit{Qwen}}
    & \multicolumn{2}{c}{\textit{Llama}} 
        & \multicolumn{2}{c}{\textit{Qwen}} \\
    \hline
    Single & \multicolumn{2}{c|}{0.33} & \multicolumn{2}{c|}{0.63} & \multicolumn{2}{c|}{0.46} & \multicolumn{2}{c|}{0.79} & \multicolumn{2}{c|}{0.79} & \multicolumn{2}{c}{0.86} \\
    Voting & \multicolumn{2}{c|}{0.43} & \multicolumn{2}{c|}{0.69} & \multicolumn{2}{c|}{0.57} & \multicolumn{2}{c|}{0.82} & \multicolumn{2}{c|}{0.92} & \multicolumn{2}{c}{0.91} \\
    \arrayrulecolor{gray!40}
    \hline
    \arrayrulecolor{black}
    &  10\% & 20\%&  10\% & 20\% &  10\% & 20\%&  10\% & 20\% &  10\% & 20\%&  10\% & 20\%  \\
    NLL-max  &  0.43 & \blue{0.54}  &  0.71 & 0.74  &  0.60 & 0.66  &  0.85 & 0.87  &  0.96 & \blue{0.97}  &  0.91 & 0.93 \\
    Entropy-max  &  0.44 & 0.51  &  0.70 & 0.71  &  0.62 & 0.70  &  0.84 & 0.84  &  0.97 & 0.95  &  0.89 & 0.91 \\
    TokUR-max  &  0.46 & 0.53  &  \blue{0.75} & 0.74  &  0.65 & 0.72  &  0.86 & \blue{0.88}  &  0.96 & 0.95  &  0.92 & \blue{0.94} \\
    \arrayrulecolor{gray!40}
    \hline
    \arrayrulecolor{black}
    Rand. Pert.  &  \blue{0.47} & \boldblue{0.56}  &  \blue{0.75} & \blue{0.76}  &  \blue{0.66} & \blue{0.75}  &  \blue{0.87} & \blue{0.88}  &  \boldblue{0.98} & \boldblue{0.98}  &  \blue{0.93} & \blue{0.94} \\
    Adv. Pert.  &  \boldblue{0.51} & \blue{0.54}  &  \boldblue{0.78} & \boldblue{0.79}  &  \boldblue{0.70} & \boldblue{0.77}  &  \boldblue{0.88} & \boldblue{0.89}  &  \boldblue{0.98} & \blue{0.97}  &  \boldblue{0.94} & \boldblue{0.95} \\
    \hline 
    \end{tabular}
    }
    \vspace{-0.2cm}
\end{table}

\textbf{Remarks.} In practice, both type of uncertain steps (wrong steps \& fragile correct steps) may coexist in the same uncertainty response. Evaluating such mixed cases is challenging due to the difference in their evaluation pipelines. Thus, we only evaluate them separately.
In Appendix~\ref{sec:appendix_cases}, we also visualize a few cases for correct responses with overall low (self-consistency) uncertainty, and our uncertainty curve usually remains flat throughout the reasoning trajectory.

\vspace{-0.3cm}
\subsection{Practical Utility}\label{exp:reflection}
\vspace{-0.2cm}

In this section, we study one possible application of our method to use our  signals to guide targeted inspection. We suppose a human practitioner who uses an LLM to solve a reasoning problem and suspects that the generated response may be unreliable. Similar to  the previous study, such unreliable responses may be identified using existing response-level uncertainty methods, i.e., self-consistency \(\leq 0.7\). By applying our method, the practitioner can narrow the response down to a small subset of potentially problematic steps. They can then inspect these flagged steps to trace possible reasoning flaws.  This can be especially useful for long responses, where inspecting every step is costly. 

In Table~\ref{tab:revision}, we simulate this inspection process using GPT-5.6. For each flagged step, GPT-5.6 is given the preceding context and asked to determine whether the flagged step is incorrectly derived, which is tasked to answer ``\emph{whether this step introduces a new error}''. If the step is judged incorrect, GPT-5.6 provides an explanation of the error and a suggestion for revision. This feedback is then provided to the base model to revise its response. Notably, if the flagged step does not introduce a ``new error'', e.g., an incorrect claim based on previous errors, we still annotate it as ``no error,'' since a human inspector would possibly regard it as consistent with the preceding context.

\textbf{Result.} From Table~\ref{tab:revision} we show the reasoning accuracy of Llama and Qwen after this revision process. We consider relatively challenging subsets from MATH and BBH, with each has 100 randomly chosen questions. Here, since we employ multiple sampling to obtain self-consistency score, we conduct the revision based on a randomly selected majority response among these samples. The results show our method can help the models answer more questions correctly, especially when only a small fraction of steps are examined (10\%) and in harder datasets such as MATH.

\vspace{-0.3cm}
\section{Additional Studies}
\vspace{-0.2cm}

We perform studies to further understand our method. More ablation studies are in Appendix~\ref{sec:ablation}.

\vspace{-0.3cm}
\subsection{How is the Efficiency?}\label{sec:efficiency} 
\vspace{-0.2cm}

\begin{wrapfigure}{r}{0.48\textwidth}
    \centering
    \vspace{-0.7cm}
    \includegraphics[width=\linewidth]{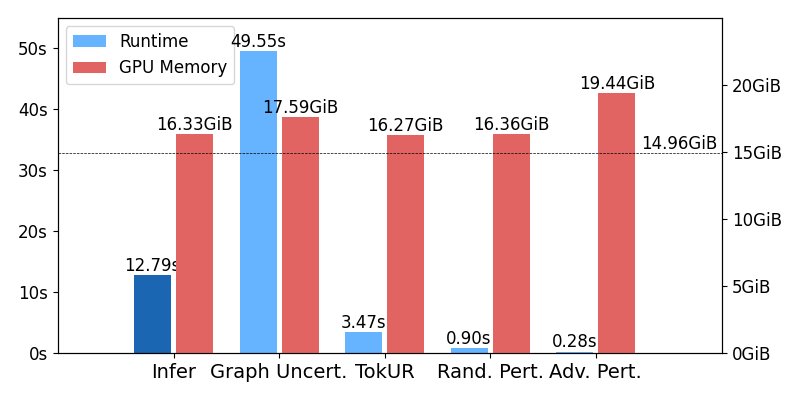}
    \vspace{-0.6cm}
    \caption{Time and memory efficiency. 
    }
    \vspace{-0.3cm}
    \label{fig:efficiency}
\end{wrapfigure}

In Figure~\ref{fig:efficiency}, 
we report the model average response generation time (deep blue), and average additional computational time (light blue) incurred by each uncertainty estimation method, together with their total memory consumption. All experiments are conducted using Llama-3.1-8B-Instruct in BF16 precision on a single NVIDIA GH200 GPU. The results show that perturbation-based metrics introduce only modest computational overhead. For example, Rand.~Pert. with $l=20$ forward passes, requires approximately 0.90s additional computation. Adv.~Pert., which involves one backward and one forward pass, incurs approximately 0.28s. Both overheads are small compared with the average 12.79s of the original response generation. In terms of memory, Adv.~Pert.~requires approximately 3 GB of additional memory due to backpropagation.

\vspace{-0.2cm}
\subsection{Hallucination Detection}\label{sec:hallucination}
\vspace{-0.2cm}

\begin{wraptable}{r}{0.45\textwidth}
\vspace{-0.8cm}
    \centering
    \caption{Hallucination Detection.}
    \vspace{-0.2cm}
    \resizebox{\linewidth}{!}{
    \begin{tabular}{c|c|c|c}
    \hline
    \textbf{FactScore} & \textit{Llama}
    & \textit{Qwen}
    & \textit{Mistral} \\
    \hline
    NLL  &  0.610  &  0.638  &  0.604 \\
    Entropy  &  0.635  &  0.623  &  0.596 \\
    TokUR  &  0.706  &  \blue{0.670}  &  0.674 \\
    CCP  &  0.660  &  0.659  &  \blue{0.687} \\
    TokenSAR  &  \blue{0.716}  &  0.631  &  0.641 \\
    Graph Uncert.  &  \boldblue{0.778}  &  \boldblue{0.878}  &  \boldblue{0.887} \\
    \hline
    Rand. Pert.  &  0.675  &  0.596  &  0.576 \\
    Adv. Pert.  &  0.678  &  0.599  &  0.586 \\
    \hline
    \end{tabular}
    }
    \label{tab:factscore}
    \vspace{-0.3cm}
\end{wraptable}

We also investigate the FactScore benchmark \citep{min2023factscore} to evaluate the effectiveness of various metrics in identifying factual errors (hallucinations) when LLMs are queried to compose celebrity biographies.
Table~\ref{tab:factscore} reports the AUROC of various methods, to compare their ability to detect hallucinated claims in the entire response. From the result, we observe that multiple-sampling-based method (Graph Uncert.) performs the best, while our method does not perform comparably. One possible explanation is that the underlying mechanism of factual hallucination can differ from the reasoning errors. For example, prior work~\citep{maynez2020faithfulness} shows that hallucinated content in an ``abstractive summarization'' task is often weakly grounded in the source context. Through human evaluation, they find that approximately 63\%--72\% of model-generated summaries contain extrinsic hallucinations, which are content that is not supported by the source material.  In this case,  these errors can be difficult to be flagged by perturbation-based methods. Admittedly, in practice, factual hallucination errors can be difficult to disentangle from reasoning errors, especially in knowledge-intensive reasoning problems~\citep{rein2024gpqa}. For example, a model may fail because it lacks knowledge about a certain entity, or because it incorrectly recalls or applies factual information. This is why we focus on relatively simpler dataset in MATH and BBH, where factual hallucination is less prevalent (although we also study in GPQA). We leave a comprehensive investigation of the interaction between hallucination and reasoning uncertainty to future work.

\vspace{-0.2cm}
\section{Conclusion and Limitation}
\vspace{-0.2cm}

Unlike prior uncertainty quantification approaches for LLM reasoning tasks that are largely confined to assessing confidence in the final answer, our method concentrates on uncertainty localization to the intermediate reasoning process. However, our study also has several limitations. \textbf{First}, our perturbation-based methods require white-box access to the model, including its token embeddings and gradients, which limits their applicability to closed-source models. \textbf{Second}, factual hallucination and reasoning errors can be difficult to disentangle in knowledge-intensive reasoning tasks. A model may fail because of missing or incorrect factual knowledge rather than instability in the reasoning process itself. We therefore focus mainly on MATH and BBH, where reasoning errors can be studied more cleanly, and leave a more systematic treatment of hallucination-related uncertainty to future work. \textbf{Third}, our method focuses on uncertainty localization rather than calibrated uncertainty quantification. The perturbation scores are designed primarily to rank relatively uncertain positions within a reasoning trajectory, rather than to provide an absolute probability of error for each step. As a result, the scores may not be directly comparable across different questions or responses.

\bibliography{iclr2027_conference}
\bibliographystyle{iclr2027_conference}

\appendix

\section{Proof}\label{app:proof}
In this section, we first provide proof for Proposition~\ref{prop:less_supportive_neighboring_condition}. For simplicity, we use $\mathbf{H}$, $\tilde{\mathbf{H}}$ and $\boldsymbol{\varepsilon}$ to replace $\mathbf{H}_{1:t-1}$, $\tilde{\mathbf{H}}_{1:t-1}$ and $\boldsymbol{\varepsilon}_{1:t-1}$, and use $\delta = \nabla_{\mathbf{H}} f_t(\mathbf{H})$.

\begin{lemma}[\textbf{Random Perturbation approximates gradient norm}]
For Random Perturbation, $\tilde{\mathbf{H}} = \mathbf{H} + \boldsymbol{\varepsilon}$. First order Taylor expansion gives $f_t(\tilde{\mathbf{H}}) = f_t(\mathbf{H} + \boldsymbol{\varepsilon}) \approx f_t(\mathbf{H}) + \delta^T\boldsymbol{\varepsilon}$ for sufficiently small $\boldsymbol{\varepsilon}$. Applying to Eq.~\ref{eq:rand_perturb}, random perturbation variance could be approximated by the square of gradient norm:

\begin{align}
    \mathrm{Pert.}(x_t) &= 
\mathrm{Var}_{\boldsymbol{\varepsilon}}
\!\left[f_t(\tilde{\mathbf{H}})
\right]\\
&\approx \mathrm{Var}_{\boldsymbol{\varepsilon}}
\!\left[f_t(\mathbf{H}) + \delta^T\boldsymbol{\varepsilon}
\right]\\
&= \mathrm{Var}_{\boldsymbol{\varepsilon}}
\!\left[\delta^T\boldsymbol{\varepsilon}
\right]\\
&= \sigma^2 \lVert \delta \rVert_2^2
\end{align}
\label{lemma:rand_pert_approx_grad_norm}
\end{lemma}

\begin{proof}[\textbf{proof for Proposition~\ref{prop:less_supportive_neighboring_condition}}]
    First, we use First order Taylor expansion $f_t(\tilde{\mathbf{H}}) - f_t(\mathbf{H}) \approx \delta^T\boldsymbol{\varepsilon}$ (again, we assume that $\boldsymbol{\varepsilon}$ is sufficiently small), and $\boldsymbol{\varepsilon} \sim \mathcal N(0, \sigma^2 I)$. Therefore $\delta^T\boldsymbol{\varepsilon} \sim \mathcal N(0, \sigma^2 \lVert \delta \rVert_2^2)$. Consequently, using the cumulative distribution function of a Gaussian distribution:
    \begin{align}
        \Pr_{\boldsymbol{\varepsilon}}
        \left[
        f_t(\tilde{\mathbf H})
        \le
        f_t(\mathbf H)-c
        \right]
        &=
        \Pr_{\boldsymbol{\varepsilon}}
        \left[
        f_t(\tilde{\mathbf H}) - f_t(\mathbf H)
        \le
        -c
        \right] \\
        &\approx
        \Pr_{\boldsymbol{\varepsilon}}
        \left[
        \delta^T\boldsymbol{\varepsilon}
        \le
        -c
        \right] \\
        &=
        \Phi\left(
        -\frac{c}{\sigma \lVert \delta \rVert_2}
        \right)
    \end{align}
    Substituting Lemma~\ref{lemma:rand_pert_approx_grad_norm}, we get:
    \begin{align}
        \Pr_{\boldsymbol{\varepsilon}}
        \left[
        f_t(\tilde{\mathbf H})
        \le
        f_t(\mathbf H)-c
        \right]
        &\approx
        \Phi\left(
        -\frac{c}{\sigma \lVert \delta \rVert_2}
        \right) \\
        &\approx
        \Phi\left(
        -\frac{c}{\sqrt{\mathrm{Pert.}(x_t)}}
        \right)
    \end{align}
    
    Hence, the proposition is proved.
    
\end{proof}

Similarly, Proposition~\ref{prop:adv_directional} can be easily proven by a first order Taylor expansion, we also use $\hat{\mathbf H}$ to replace $\hat{\mathbf H}_{1:t-1}$ for simplicity.

\begin{proof}[\textbf{proof for Proposition~\ref{prop:adv_directional}}]

Since the adversarial perturbation is sufficiently small, we use first order Taylor expansion to expand $f_t(\hat{\mathbf H}) = \log \Pr(x_t \mid \hat{\mathbf H})$:
\begin{align}
    f_t(\hat{\mathbf H}) &= f_t(\mathbf H-
\alpha\operatorname{sgn} (\nabla_{\mathbf H}\mathcal{F}(\mathbf H))) \\
&\approx f_t(\mathbf H) - \alpha \cdot \langle \nabla_{\mathbf{H}_{1:t-1}} \log \Pr(x_t \mid \mathbf{H}_{1:t-1}) ~,~ \operatorname{sgn} (\nabla_{\mathbf H_{1:t-1}}\mathcal{F}(\mathbf H)) \rangle
\end{align}

Substitute back to the definition of adversarial perturbation:

\begin{align}
    \mathrm{Pert.}(x_t)
&=
f_t(\mathbf H_{1:t-1})
-
f_t(\hat{\mathbf H}_{1:t-1}) \\
&\approx \alpha \cdot \langle \nabla_{\mathbf{H}_{1:t-1}} \log \Pr(x_t \mid \mathbf{H}_{1:t-1}) ~,~ \operatorname{sgn} (\nabla_{\mathbf H_{1:t-1}}\mathcal{F}(\mathbf H)) \rangle
\end{align}

Hence, the proposition is proved.

\end{proof}

\section{Prompts and Human Examination}\label{sec:prompts}

\textbf{Prompts.}
We use a strong online LLM `GPT-5.6-sol' for annotation works in the experiment. The prompts are provided in Figure~\ref{fig:prompts}. The first prompt is applied for judging if a reasoning response contains intermediate mistakes in Section~\ref{exp:incorrect}. The second prompt applied for locating first incorrect step in Section~\ref{exp:incorrect}. The third prompt applied in Section~\ref{exp:correct} to determine the correctness of step regeneration in order to calculate risk score. The last prompt applied in Section~\ref{exp:reflection} is used for feedback and revision.

\textbf{Human Examination of GPT Annotation.} 
For Section~\ref{exp:incorrect}, we 
first focus on 180 responses which 
GPT annotated as ``containing incorrect reasoning steps''. We employ human inspector to examine the quality of GPT's annotation in the task of finding ``first wrong step''. In this procedure, we ask the human inspector to read the first wrong step suggested by the GPT, and check the preceding reasoning contents to find whether the preceding contents have earlier errors. If human inspector agree with GPT judgement, we label it as an \textbf{aligned case}.
As shown in Table~\ref{tab:human_examination}, these annotations achieve an overall alignment rate of 93.9\%. Meanwhile, the human inspector find all these 180 cases indeed contain wrong steps.

On the other hand, we also let the human inspector to read through the responses that GPT find no reasoning errors, and meanwhile referring to the correct answer provided by the dataset. Among the 30 examined cases, 29 cases are verified to be no errors. The remaining case contains ambiguity, making its reasoning correctness difficult to determine conclusively.

Finally, for Section~\ref{exp:correct}, we manually examine 80 responses for the annotation task that check whether a re-generated step is correct in risk calculation. It shows 97.5\% alignment rate. These results indicate that GPT provides sufficiently reliable annotations for the evaluations considered in our experiments.

\begin{table}[ht]
    \centering
    \caption{Human examination result for `first wrong step' annotation in 180 responses.}
    \label{tab:human_examination}
    
    \resizebox{\linewidth}{!}{
    \begin{tabular}{c|cccccc}
    \hline

    & \multicolumn{6}{c}{\textbf{Math}}\\
    &  algebra & cnt. \& prob. & number theory & precalculus & inter. algebra & prealgebra \\
    \hline
    \textit{Llama}  &  5 / 5 & 4 / 5 & 5 / 5 & 4 / 5 & 5 / 5 & 5 / 5 \\
    \textit{Qwen}  &  4 / 5 & 4 / 5 & 4 / 5 & 4 / 5 & 5 / 5 & 5 / 5 \\
    \textit{Mistral}  &  4 / 5 & 5 / 5 & 4 / 5 & 5 / 5 & 5 / 5 & 5 / 5 \\
    
    \hline

    & \multicolumn{6}{c}{\textbf{Logic}}\\
    &  logical deduction & tracking shuffled & web of lies & formal fallacies & navigate & temp. seq. \\
    \hline
    \textit{Llama}  &  5 / 5 & 5 / 5 & 5 / 5 & 5 / 5 & 5 / 5 & 5 / 
    5 \\
    \textit{Qwen}  &  5 / 5 & 5 / 5 & 5 / 5 & 5 / 5 & 5 / 5 & 4 / 5 \\
    \textit{Mistral}  &  4 / 5 & 5 / 5 & 5 / 5 & 5 / 5 & 4 / 5 & 5 / 5 \\
    \hline
    \end{tabular}
    }

\end{table}

\begin{figure*}[hb]
    \centering
    \begin{minipage}{0.7\linewidth}
        \centering
            \centering
            \includegraphics[width=\linewidth]{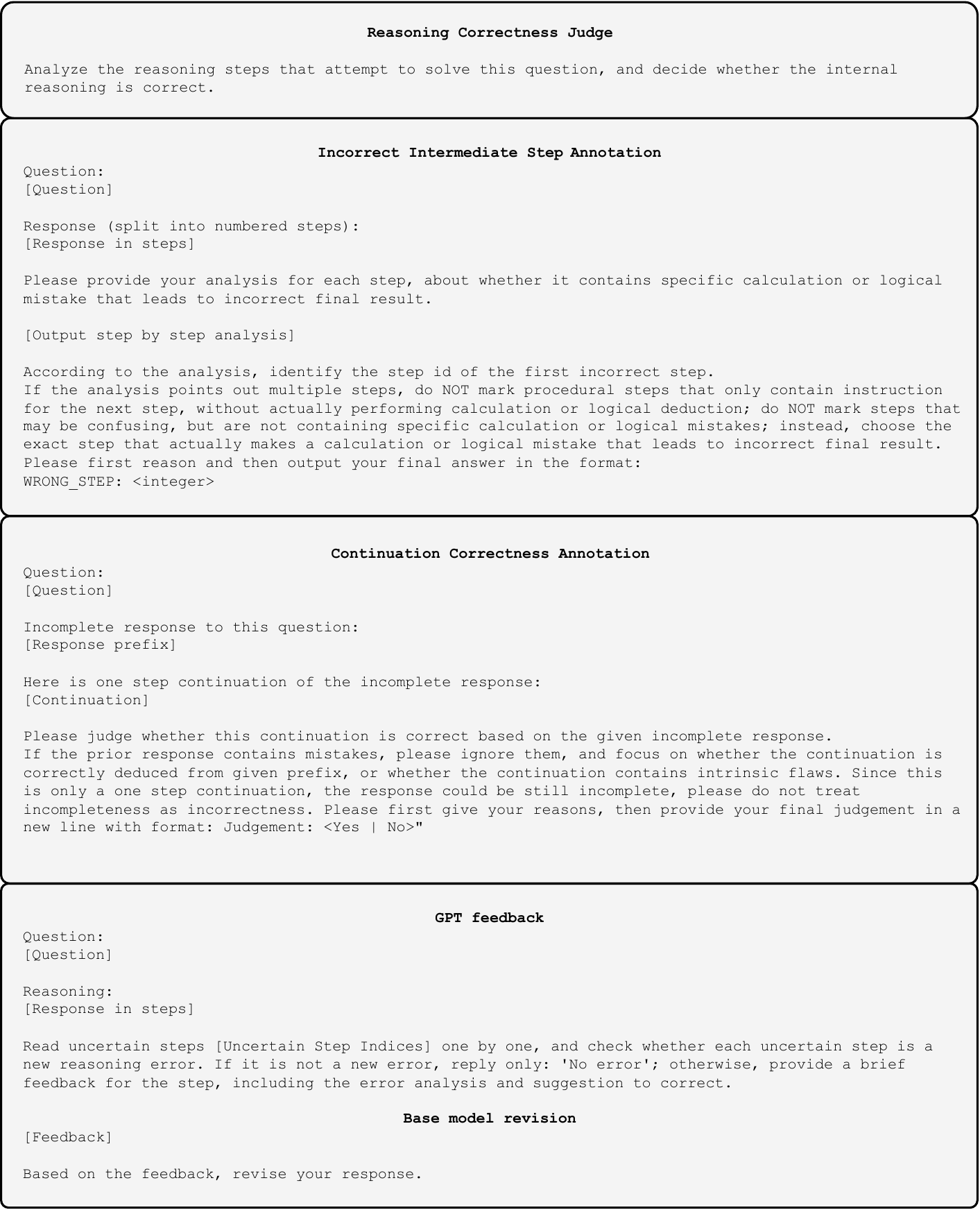}
        \end{minipage}
        \caption{Prompts used in the experiments. Title of each prompt is in bold, placeholders are marked as `[]' with description within.}
    \label{fig:prompts}
\end{figure*}

\section{Detailed Experiment Setup}\label{sec:exp_detail}

In this section we provide more detailed information about our experiments in Section~\ref{sec:main_exp}. Including benchmark subset selection; baseline settings; step split detail; and simulation details for application scenario.

\textbf{{Subsets.}}
For the datasets we use in our experiments, MATH~\citep{hendrycks2021measuring} and BBH\citep{suzgun2022challenging} contain subsets that split the math or logic reasoning into more fine-grained tasks. For MATH~\citep{hendrycks2021measuring}, we consider the testing split of 6 subsets: `algebra', `counting\_and\_probability', `number\_theory', `precalculus', `intermediate\_algebra' and `prealgebra'. We will skip questions with graphics. For BBH\citep{suzgun2022challenging}, we consider the testing split of 10 subsets: `logical\_deduction\_three\_objects', `logical\_deduction\_five\_objects', `logical\_deduction\_seven\_objects', `tracking\_shuffled\_objects\_three\_objects', `tracking\_shuffled\_objects\_five\_objects', `tracking\_shuffled\_objects\_seven\_objects', `web\_of\_lies', `formal\_fallacies', `navigate' and `temporal\_sequences'.

\textbf{{Baselines.}} In our evaluation, we compare the proposed method with various intermediate uncertainty baselines. For baselines that provide token-wise signals, we have Negative Log-Likelihood (\textbf{NLL}), token entropy (\textbf{Entropy}), and \textbf{TokUR}(EU)~\citep{zhang2026tokur}. TokUR is a Bayesian-based token-level uncertainty estimator with low-rank weight perturbation, but the original evaluation is primarily conducted after aggregating those token-level uncertainties into response-level scores, rather than evaluating in intermediate steps. We present the epistemic uncertainty (EU) version of TokUR for reasoning because it performs better than total uncertainty (TU) or aleatoric uncertainty (AU) in our setting. Except for FactScore~\citep{min2023factscore} we use AU which is because it performs the best, which is also supported by the original paper. For hyperparameters of TokUR, after exploration, we find perturbation strength $\sigma_q = 0.001$, rank $r' = 8$ and number of samples $M = 2$ is the optimal setting in our main experiment of Section~\ref{sec:main_exp}. In addition, we have \textbf{CCP}~\citep{fadeeva2024fact}, a probability-based method that removes the impact of alternative same-meaning tokens and functional words. As it treats functional words as $\text{CCP}_{\text{word}} = 1$, and it is only quantitatively evaluated in response-level, we will use the sequence aggregated version provided by the original paper. \textbf{TokenSAR}~\citep{duan2023shifting} is another probability-based method that weight negative log-likelihood with semantic significancy to predict uncertainty for a token. To perform its essence of "weighted uncertainty" or "shifted uncertainty", we will directly use its own sequence aggregated version. Graph Uncertainty(\textbf{Graph~Uncert.})~\citep{jiang2024graph} is a Self-Consistency\citep{wang2022self} like multiple sampling based claim-level uncertainty metric, that split sampled responses and their claims into bipartite graph, represent uncertainty using graph features. We will use closeness ($C_C$), as it is presented as the best metric in the original experiment. This method is originally designed for factual claim uncertainty estimation, and it's workflow is combined into FactScore~\citep{min2023factscore} dataset. Therefore, we have made modifications to migrate this method to evaluate reasoning step consistency. Specifically, we use ``sentence-transformers/all-MiniLM-L6-v2'' to produce sentence embedding for each step, then use KMeans to cluster the steps into nodes, we will then follow the Graph Uncertainty methodology to obtain the uncertainty scores for each node.

\textbf{Step split detail.} For an LLM response, we use an efficient rule-based step segmentation method. We use sentence-ending indication characters `.!?\textbackslash n' as step boundaries, and handle the following corner cases: 1. '.' serving as decimal point or ellipsis will not split; 2. '!' as factorial (next to digit) will not split; 3. brackets, '\textbackslash n\textbackslash t' following a step will be included in that sentence (e.g. 'Claire is now left midfielder.\textbackslash n\textbackslash n' is one step); 4. step shorter than 5 chars will not split and be included to later step; 5. whatever is left at the end becomes the last step. Token boundaries are mostly overlap with step boundaries, if not, it will be mapped to the later step.

\textbf{Simulation details for application scenario.} 
To simulate the application scenario in Section~\ref{exp:reflection}, we sample $n=10$ generations for each problem and randomly select one generation whose final answer matches the majority answer, i.e., the answer with the highest self-consistency score, as the target response. The accuracy of these target responses constitutes the majority-vote baseline (correspondingly, `single' baseline simply uses the first generation as the target response). Responses with a consistency score of more than 0.7 are considered reliable and are returned directly. For the remaining unreliable responses, each UQ method flags the top 10\% or 20\% most uncertain reasoning steps. The exact number of selected steps is obtained by rounding 10\% or 20\% of the total number of steps in the response, with at least one step selected (the same rule is used for top-10\% and top-20\% settings in all experiments). We then use GPT-5.6 as an external assistant to provide feedback on each flagged step. The base model revises its response based on this feedback, and the revised response is taken as the final output. Accuracy over all final outputs, including both directly returned reliable responses and revised unreliable responses, measures the practical utility of each UQ method in this scenario.

\begin{table}[t]
    \centering
    \caption{Performance of each method in 300 responses with incorrect final answer, The 95\% conf.~interval is $\approx\pm 0.057$ for 300 Bernoulli trials (p=0.5).}
    \vspace{-0.2cm}
    \label{tab:incorrect_final}
    \resizebox{\linewidth}{!}{

    \begin{tabular}{c|cc|cc|cc|cc|cc|cc}
    \hline 
    &     \multicolumn{6}{c|}{\textbf{Math}} &  \multicolumn{6}{c}{\textbf{Logic}}
    \\
    & \multicolumn{2}{c}{\textit{Llama}}
    & \multicolumn{2}{c}{\textit{Qwen}}
    & \multicolumn{2}{c|}{\textit{Mistral}}
    & \multicolumn{2}{c}{\textit{Llama}}
    & \multicolumn{2}{c}{\textit{Qwen}}
    & \multicolumn{2}{c}{\textit{Mistral}}\\
    \hline
    &  10\%  & 20\%  &  10\%  & 20\%  &  10\%  & 20\%
    &  10\%  & 20\%  &  10\%  & 20\%  &  10\%  & 20\%
    \\
    NLL-max  &  0.18 & 0.29  &  0.21 & 0.39  &  0.15 & 0.28  &  0.21 & 0.39  &  0.28 & 0.51  &  0.17 & 0.33 \\
    Entropy-max  &  0.19 & 0.33  &  0.24 & 0.41  &  0.14 & 0.24  &  0.21 & 0.40  &  0.26 & 0.51  &  0.15 & 0.31 \\
    TokUR-max  &  0.27 & 0.49  &  0.26 & 0.44  &  0.28 & 0.44  &  0.27 & 0.42  &  0.29 & 0.45  &  0.26 & 0.45 \\
    \arrayrulecolor{gray!40}
    \hline
    \arrayrulecolor{black}
    NLL-avg  &  0.13 & 0.24  &  0.17 & 0.31  &  0.09 & 0.18  &  0.18 & 0.38  &  0.19 & 0.39  &  0.11 & 0.27 \\
    Entropy-avg  &  0.13 & 0.26  &  0.15 & 0.34  &  0.10 & 0.18  &  0.19 & 0.41  &  0.19 & 0.40  &  0.12 & 0.26 \\
    TokUR-avg  &  0.20 & 0.36  &  0.18 & 0.34  &  0.24 & 0.42  &  0.24 & 0.40  &  0.23 & 0.41  &  0.25 & 0.38 \\
    \arrayrulecolor{gray!40}
    \hline
    \arrayrulecolor{black}
    CCP  &  0.19 & 0.34  &  0.31 & \blue{0.50}  &  0.20 & 0.34  &  0.26 & 0.45  &  0.31 & \blue{0.55}  &  0.20 & 0.38 \\
    TokenSAR  &  0.18 & 0.27  &  0.26 & 0.42  &  0.18 & 0.31  &  0.27 & 0.46  &  0.30 & 0.53  &  0.22 & 0.40 \\
    Graph Uncertain.  &  0.20 & 0.28  &  0.20 & 0.31  &  0.15 & 0.20  &  0.22 & 0.36  &  0.30 & 0.44  &  0.15 & 0.30 \\
    \arrayrulecolor{gray!40}
    \hline
    \arrayrulecolor{black}
    Rand.~Pert.  &  \blue{0.40} & \blue{0.54}  &  \blue{0.32} & 0.49  &  \blue{0.32} & \boldblue{0.51}  &  \blue{0.34} & \blue{0.50}  &  \blue{0.39} & 0.54  &  \blue{0.28} & \blue{0.48} \\
    Adv.~Pert.  &  \boldblue{0.45} & \boldblue{0.60}  &  \boldblue{0.35} & \boldblue{0.52}  &  \boldblue{0.35} & \boldblue{0.51}  &  \boldblue{0.43} & \boldblue{0.58}  &  \boldblue{0.40} & \boldblue{0.57}  &  \boldblue{0.36} & \boldblue{0.50} \\
    \arrayrulecolor{black}
    \hline 
    \end{tabular}
    }
    
\end{table}

\begin{table}[t]
    \centering
    \caption{Results on MATH and BBH, following the same setting in Section~\ref{exp:incorrect}, with LLM temperature $T=0.5$.}
    \vspace{-0.2cm}
    \label{tab:incorrect_t05}
    \resizebox{\linewidth}{!}{

    \begin{tabular}{c|cc|cc|cc|cc|cc|cc}
    \hline 
    &     \multicolumn{6}{c|}{\textbf{Math}} &  \multicolumn{6}{c}{\textbf{Logic}}
    \\
    & \multicolumn{2}{c}{\textit{Llama}}
    & \multicolumn{2}{c}{\textit{Qwen}}
    & \multicolumn{2}{c|}{\textit{Mistral}}
    & \multicolumn{2}{c}{\textit{Llama}}
    & \multicolumn{2}{c}{\textit{Qwen}}
    & \multicolumn{2}{c}{\textit{Mistral}}\\
    \hline
    &  10\%  & 20\%  &  10\%  & 20\%  &  10\%  & 20\%
    &  10\%  & 20\%  &  10\%  & 20\%  &  10\%  & 20\%
    \\
    NLL-max  &  0.20 & 0.33  &  0.21 & 0.42  &  0.17 & 0.32  &  0.20 & 0.39  &  0.26 & 0.47  &  0.14 & 0.32 \\
    Entropy-max  &  0.24 & 0.37  &  0.24 & 0.40  &  0.12 & 0.24  &  0.18 & 0.38  &  0.23 & 0.47  &  0.15 & 0.31 \\
    TokUR-max  &  0.25 & 0.46  &  0.27 & 0.45  &  0.28 & 0.45  &  0.26 & 0.40  &  0.29 & 0.49  &  0.23 & 0.45 \\
    \arrayrulecolor{gray!40}
    \hline
    \arrayrulecolor{black}
    NLL-avg  &  0.17 & 0.27  &  0.16 & 0.33  &  0.11 & 0.20  &  0.17 & 0.38  &  0.19 & 0.37  &  0.10 & 0.27 \\
    Entropy-avg  &  0.17 & 0.29  &  0.13 & 0.34  &  0.10 & 0.20  &  0.18 & 0.40  &  0.17 & 0.37  &  0.12 & 0.27 \\
    TokUR-avg  &  0.20 & 0.35  &  0.17 & 0.35  &  0.24 & 0.40  &  0.22 & 0.40  &  0.25 & 0.45  &  0.22 & 0.41 \\
    \arrayrulecolor{gray!40}
    \hline
    \arrayrulecolor{black}
    CCP  &  0.22 & 0.37  &  \blue{0.32} & 0.48  &  0.21 & 0.33  &  0.25 & 0.43  &  0.28 & 0.52  &  0.20 & 0.38 \\
    TokenSAR  &  0.23 & 0.33  &  0.28 & 0.41  &  0.19 & 0.32  &  0.27 & 0.46  &  0.27 & 0.51  &  0.23 & 0.41 \\
    Graph Uncertain.  &  0.18 & 0.26  &  0.18 & 0.28  &  0.14 & 0.18  &  0.23 & 0.37  &  0.30 & 0.43  &  0.16 & 0.29 \\
    \arrayrulecolor{gray!40}
    \hline
    \arrayrulecolor{black}
    Rand.~Pert.  &  \blue{0.36} & \blue{0.51}  &  0.31 & \blue{0.50}  &  \blue{0.31} & \blue{0.48}  &  \blue{0.32} & \blue{0.50}  &  \blue{0.38} & \blue{0.55}  &  \blue{0.27} & \blue{0.46} \\
    Adv.~Pert.  &  \boldblue{0.39} & \boldblue{0.56}  &  \boldblue{0.33} & \boldblue{0.51}  &  \boldblue{0.33} & \boldblue{0.49}  &  \boldblue{0.43} & \boldblue{0.59}  &  \boldblue{0.41} & \boldblue{0.57}  &  \boldblue{0.35} & \boldblue{0.50} \\
    \arrayrulecolor{black}
    \hline 
    \end{tabular}
    }
    
\end{table}

\begin{table}[hb]
    \centering
    \caption{Experiment results using different system prompts.}
    \label{tab:sys_prompts}

    \begin{subtable}{\linewidth}
        \centering
        \resizebox{0.6\linewidth}{!}{
            \begin{tabular}{|c|c|}
                \hline
                Prompt 1 (default) & You are a helpful assistant. \\
                \hline
                Prompt 2 & You are a careful and rigorous reasoner. \\
                \hline
                Prompt 3 & You are an expert problem solver in reasoning. \\
                \hline
            \end{tabular}
        }
    \end{subtable}
    \vspace{0.3cm}

    \begin{subtable}{\linewidth}
        \centering
        \resizebox{0.6\linewidth}{!}{
    \begin{tabular}{c|cc|cc|cc}
    \hline
    & \multicolumn{2}{c|}{Prompt 1} & \multicolumn{2}{c|}{Prompt 2} & \multicolumn{2}{c}{Prompt 3} \\
    \hline
    &  10\%  & 20\%  &  10\%  & 20\%  &  10\%  & 20\% \\
    NLL  &  0.18 & 0.34  &  0.22 & 0.32  &  0.20 & 0.32 \\
    Entropy  &  0.24 & 0.32  &  0.18 & 0.32  &  0.20 & 0.36 \\
    TokUR  &  0.28 & 0.46  &  0.30 & 0.48  &  0.26 & 0.48 \\
    CCP  &  0.22 & 0.32  &  0.20 & 0.32  &  0.18 & 0.32 \\
    TokenSAR  &  0.22 & 0.28  &  0.20 & 0.32  &  0.20 & 0.32 \\
    Graph Uncertain.  &  0.20 & 0.30  &  0.22 & 0.24  &  0.20 & 0.34 \\
    \hline
    Rand. Pert.  &  \blue{0.40} & \blue{0.56}  &  \blue{0.38} & \blue{0.56}  &  \blue{0.40} & \blue{0.62} \\
    Adv. Pert.  &  \boldblue{0.46} & \boldblue{0.62}  &  \boldblue{0.46} & \boldblue{0.60}  &  \boldblue{0.44} & \boldblue{0.66} \\
    \hline
    \end{tabular}
    }
    \end{subtable}

\end{table}

\vspace{-0.2cm}
\section{Extra Experiments and Results}\label{app:extra_result}
\vspace{-0.2cm}

We provide additional experiment results in this section.

\textbf{MATH and BBH incorrect step locating on responses with wrong final answer.} In Table~\ref{tab:incorrect} we sample 300 responses containing incorrect reasoning steps. Some of them result in correct final answer. To erase this influence, we continue to sample based on previous results and until 300 responses with incorrect final answer is collected, then measure the performance of each method in Table~\ref{tab:incorrect_final}. The similar observations could be drawn from this table.

\textbf{MATH and BBH incorrect step locating with $T=0.5$.}
Table~\ref{tab:incorrect_t05} reports the successful detection rate of the first incorrect reasoning step for various uncertainty methods at a model temperature of $T=0.5$, as a comparison for our main experiment setting $T=0.2$. We collect a total of 120 cases for each setting from MATH and BBH. The results maintain similar, leading to conclusions consistent with those discussed in Section~\ref{sec:main_exp}.

\textbf{Experiments using different system prompts}. Table~\ref{tab:sys_prompts} record the result of MATH `precalculus' subset using Llama, under 3 versions of system prompts for querying the model with reasoning questions. The results exhibit similar behavior across all three prompts, indicating that the choice of system prompt does not materially affect the observed results.

\begin{figure}[h]
    \centering
\includegraphics[width=0.6\linewidth]{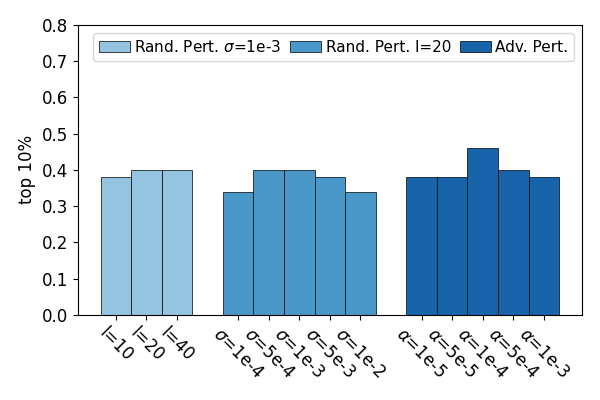}
    \vspace{-0.2cm}
    \caption{\small Result under various hyper-parameters.}
    \label{fig:ablation}
\end{figure}

\begin{figure}[bh]
    \centering
    \includegraphics[width=0.5\linewidth]{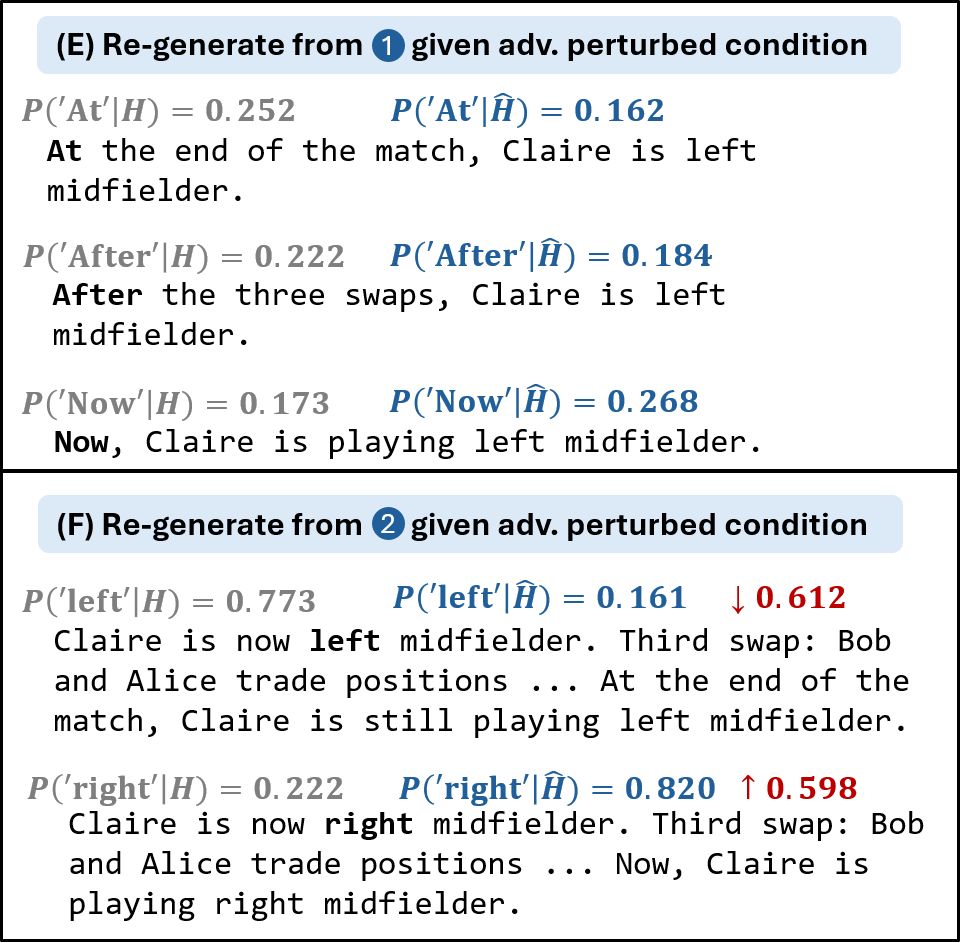}
    \caption{For the same setting as Figure~\ref{fig:legitimate}, we re-generate continuations conditioned on adversarial perturbed embeddings $\hat{\mathbf{H}}_{1:t-1}$. We find that the continuations are also coherent and on-task, similar to Figure~\ref{fig:legitimate}. This shows that adversarial perturbation on the embedding also does not disrupt reasoning coherence.}
    \label{fig:legitimate_adv}
\end{figure}

\textbf{Hyperparameter Study}\label{sec:ablation}
In our experiments, we set the hyperparameters as follows by default: (1) the number of noise sampling in Rand.~Pert.~to be $l = 20$ times. (2) The standard deviation of the noise $\sigma$ in Rand.~Pert.~to be 0.001. (3)
The intensity value $\alpha$ in Adv.~Pert.~to be 0.0001. We select the value of $l$ as a trade-off in efficiency, and we choose perturbation scales ($\sigma$ and $\alpha$) to be sufficiently small. In this subsection, we conduct ablation study to see how different values of these hyperparameters affect the performance. The experiment is conducted on MATH precalculus subset using Llama, the result is shown in Figure~\ref{fig:ablation}. For the hyperparameter $l$, there is minor increase from $l=10$ to $l=20$, and almost no increase from $l=20$ to $l=40$. This may suggest the marginal gain for increasing $l$ is diminishing. For noise scale $\sigma$, we find that our selected $\sigma=0.001$ serves as a reasonable hyperparameter value. Similar observation also exists in $\alpha=0.0001$ for Adv.~Pert. These results confirm that our selected hyperparameters in the main results are suitable, and the performance is not very sensitive to the selection under a reasonable range.

\begin{figure}[hb]
    \centering
    \includegraphics[width=\linewidth]{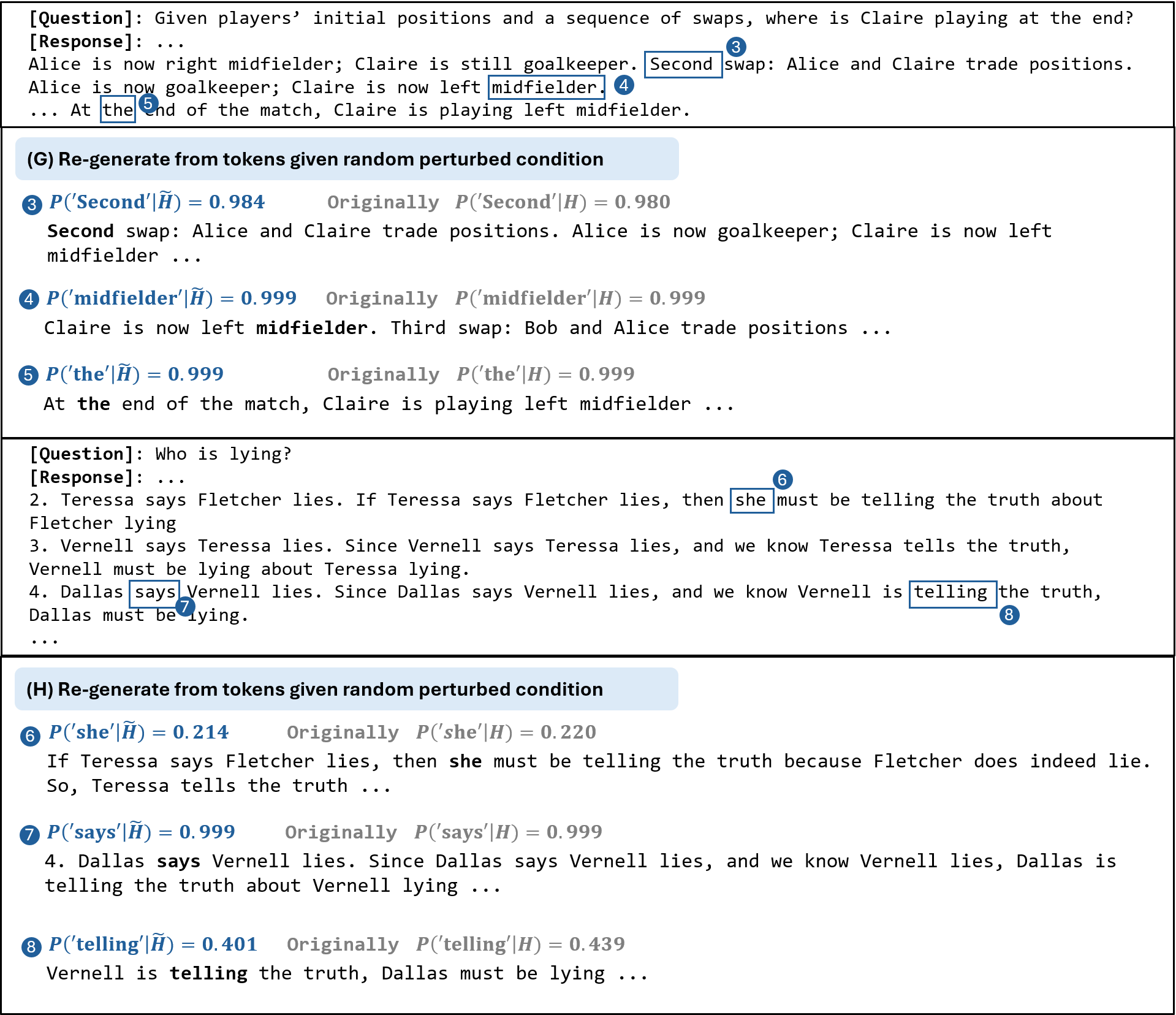}
    \caption{Extra cases of regeneration from perturbed conditions.}
    \label{fig:legitimate_extra}
\end{figure}

\section{Case Studies}\label{sec:appendix_cases}

\textbf{Extra cases of re-generation from perturbed conditions.}
Figure~\ref{fig:legitimate_extra} shows extra re-generations from random perturbed conditions remain valid and coherent, such that small embedding perturbation does not disrupt reasoning coherence.

\textbf{Cases of correct and low risk solutions.} Figure~\ref{fig:case_certain} compares three certain responses (with self-consistency \(> 0.7\)) and three uncertain responses (containing reasoning errors) by Llama in MATH. The scores are presented without normalization to facilitate a direct comparison of their absolute values. In general, the uncertainty captured by perturbation-based methods remain relatively low and stable for most tokens, with a few tokens exhibiting substantially higher scores. Notably, uncertain responses tend to exhibit higher peaks, suggesting that these localized high-score tokens may serve as indicators of increased uncertainty. Certain responses performs mostly flat and low peak, showing that the outputs are indeed certain and stable.

\begin{figure}[bh]
\centering
    \begin{minipage}{0.95\linewidth}
        \includegraphics[width=\linewidth]{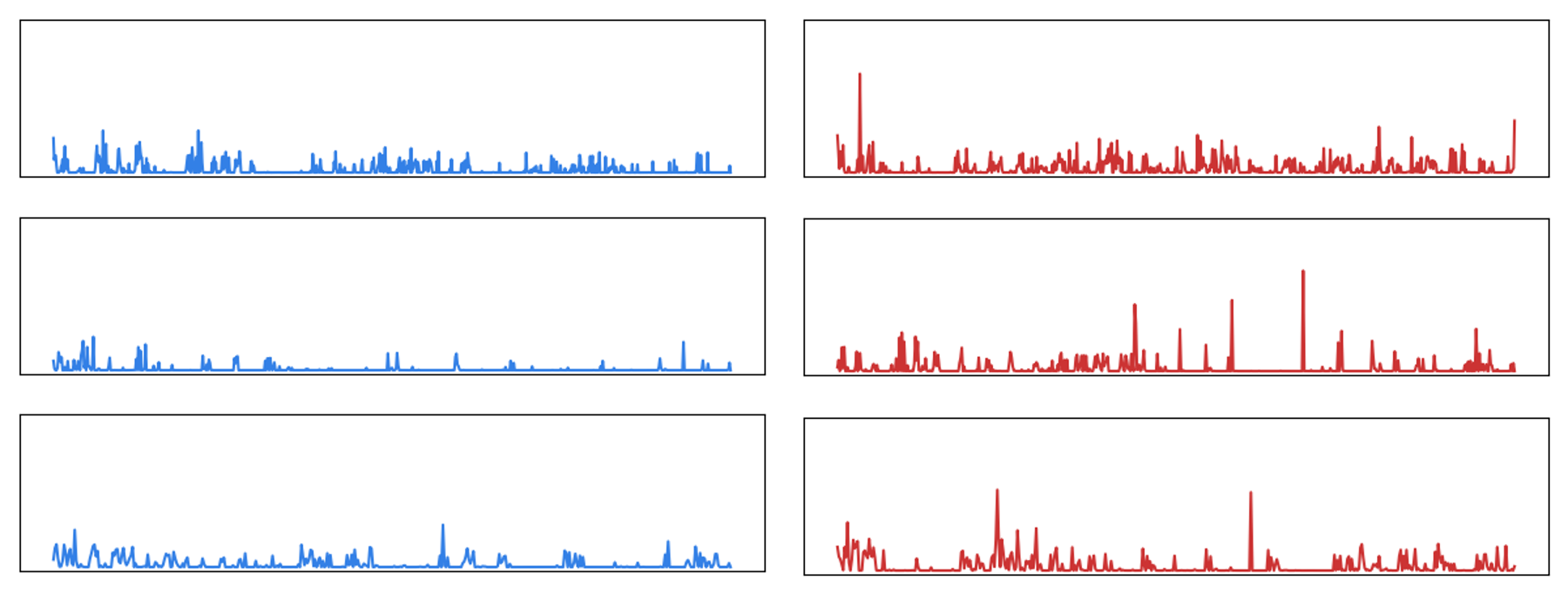}
        {\small Un-normalized \textbf{random perturbation} score curves of 3 certain and 3 uncertain responses.}

        \includegraphics[width=\linewidth]{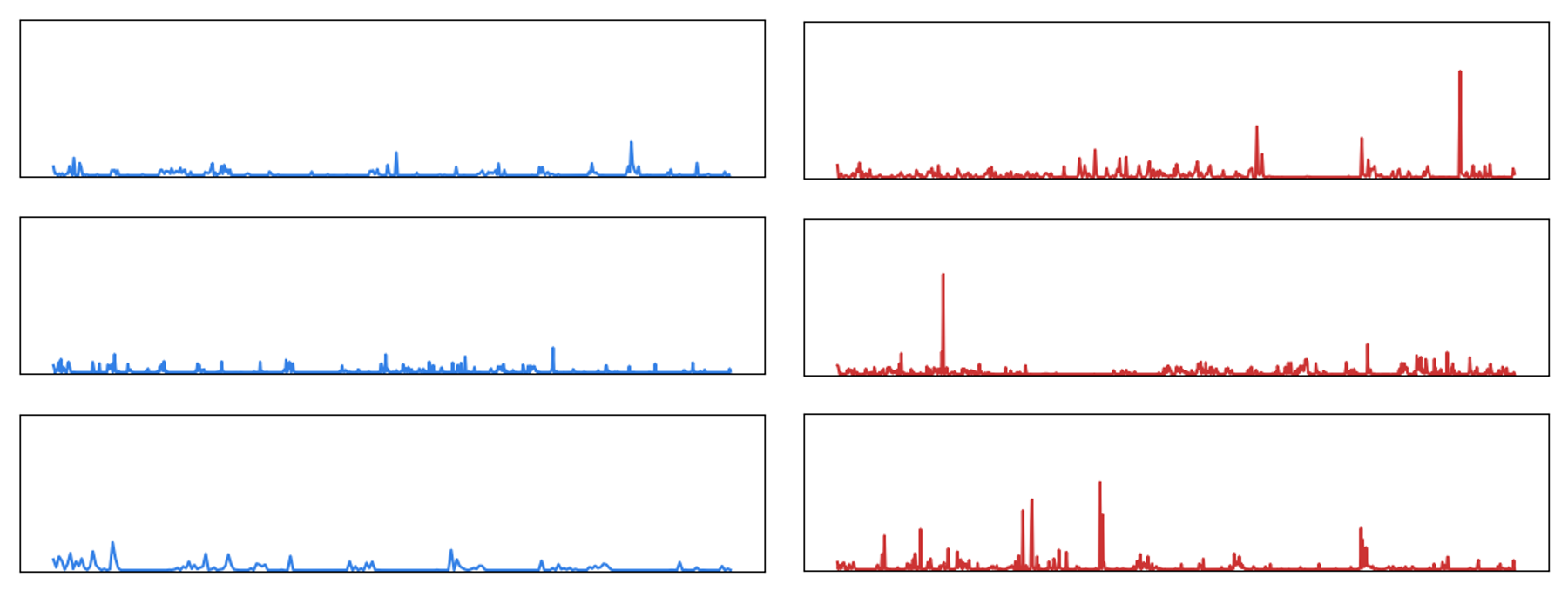}
        {\small Un-normalized \textbf{adversarial perturbation} score curves of 3 certain and 3 uncertain responses.}
    \end{minipage}
    \caption{\small Comparison of un-normalized perturbation-based scores between uncertain responses (red) and certain responses (blue). Here certain responses indicate self-consistency \(> 0.7\).}
    \label{fig:case_certain}
\end{figure}

\textbf{More cases under various settings.}
In Figure~\ref{fig:extra_math_case} and Figure~\ref{fig:extra_logic_case}, we present more cases under different datasets and models. In Figure~\ref{fig:extra_math_case}, we collect cases from MATH dataset using Llama, Qwen, and Mistral. In Figure~\ref{fig:extra_logic_case}, we collect cases from BBH dataset using Llama, Qwen, and Mistral. These cases provide more support to Section~\ref{sec:main_exp}, that perturbation-based methods better identify erroneous tokens related to reasoning instead of semantics, and they are able to identify multiple forms of errors by flagging their uncertainty.

\textbf{{Success and failure cases.}} In Figure~\ref{fig:main_case}, we provide two examples in logic reasoning, one successful case by Llama and one failure case by Qwen. As shown in Llama case, the model wrongly recalls a previous conclusion, and claims ``Vernell is telling the truth''. The perturbation-based metrics can precisely locate the token ``telling'' where the model begins to deviate.
In the Qwen case, the model incorrectly infers ``second newest is tractor'' from the preceding fact ``the tractor is the second oldest.'' Although the perturbation-based metrics assign a locally higher uncertainty score to the first incorrect token `Tr' than to its neighboring tokens, several other tokens, such as `ranking' and `Newest' receive even higher scores. This case visualizes a confident error, that our methods fails to identify.

\textbf{Case of correct but risky step.}
Sometimes we observe uncertainty in correct steps, which may suggest that the step is risky (Section~\ref{exp:correct}). Figure~\ref{fig:risk_case} shows a case where the model generates a critical numeric token `34' correctly but uncertain (high perturbation scores). Regeneration from the start of this step further shows that the token `34' is indeed unstable, as it contains risk to be replaced by incorrect tokens such as `32' or `38'.

\begin{figure}[ht]
\centering
    \begin{minipage}{0.95\linewidth}
        \includegraphics[width=\linewidth]{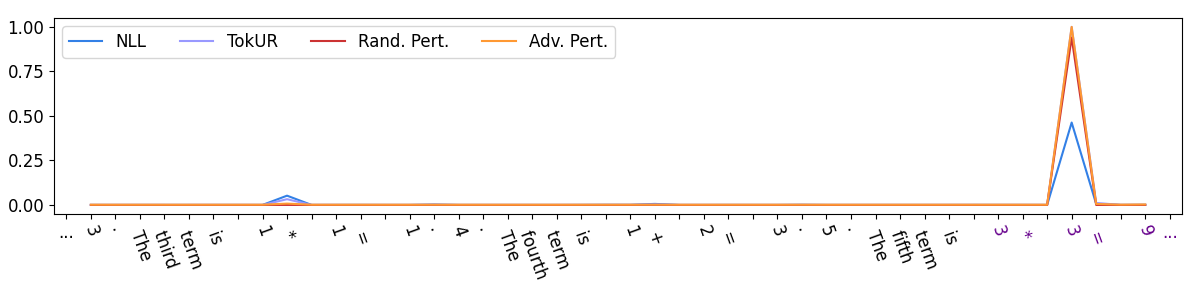}
        {\small \textbf{$\bullet$Problem misread error identified in Llama.} The LLM is asked to construct a sequence, for which the rule of this sequence is given in the problem. According to the rule, the fifth term should be $3 \times 2 = 6$, it misread the rule and incorrectly generates \textcolor{violet}{$3 \times 3 = 9$}.}

        \includegraphics[width=\linewidth]{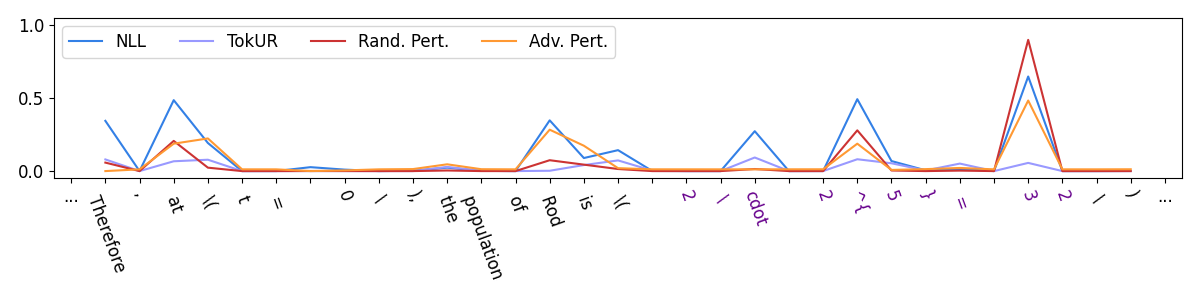}
        {\small \textbf{$\bullet$Calculation error identified in Qwen.} The model mistakenly calculates \textcolor{violet}{`$2 \cdot 2^5 = 32$'}, where token `\textcolor{violet}{3}' is successfully identified. In this case adversarial perturbation maintains a uniformly high value, and its reaction to error token is less significant than random perturbation, although it is still in one of the top-3 uncertain tokens across the response.}

        \includegraphics[width=\linewidth]{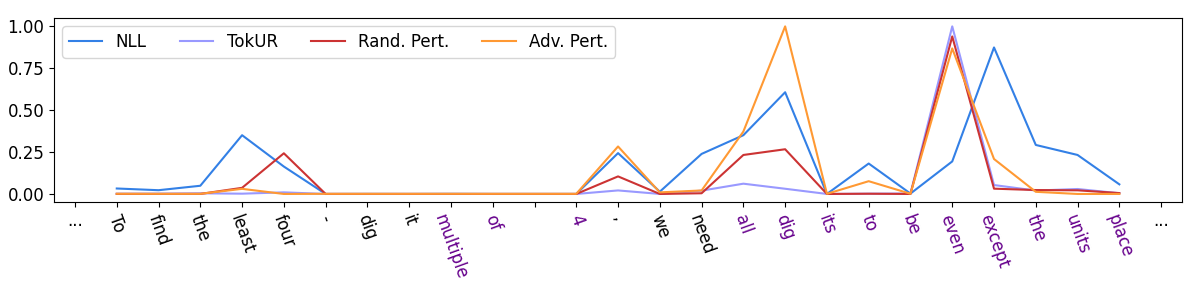}
        {\small \textbf{$\bullet$Math conceptual error identified in Mistral.} The model need the condition of digits \textcolor{violet}{`multiple of 4'}, it incorrectly state \textcolor{violet}{`all digits to be even except the units place'}, which should be replaced by `last two digits divisible by 4'. Perturbation-based methods successfully flagged critical tokens \textcolor{violet}{`all dig(its)'} and \textcolor{violet}{`even'}.}
    \end{minipage}
    \caption{Extra case studies on math reasoning task. Each uncertainty score is min-max normalized over the complete output to fit the figure, and only the critical segments are shown.}
    \label{fig:extra_math_case}
\end{figure}

\begin{figure}[t]
\centering
    \begin{minipage}{0.95\linewidth}
        \includegraphics[width=\linewidth]{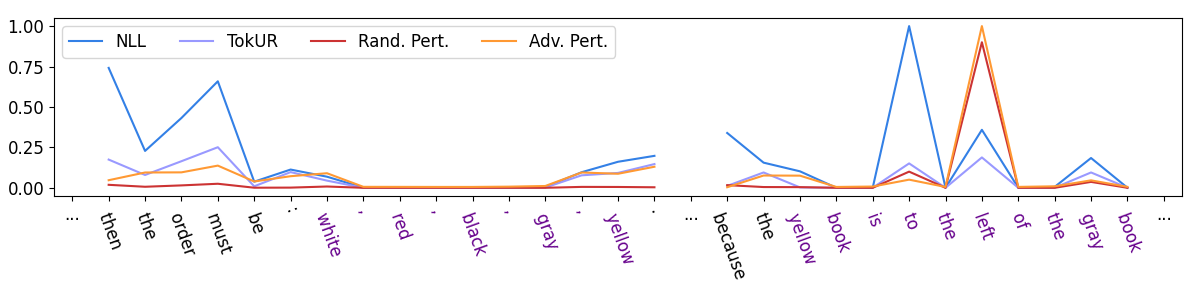}
        {\small \textbf{$\bullet$ Premise inconsistency error identified in Llama.} The LLM is tasked to sort 5 books in different colors. In the middle of the reasoning, LLM has already derived the correct order for the book: \textcolor{violet}{`white, red, black, gray, yellow'}. But later, in the answer selection stage, it incorrectly output \textcolor{violet}{`the yellow book is to the left of the gray book'}(the correct token should be `right'). For another uncertainty token `to', there may be various ways to express (e.g. `on the left side'), which may cause a high NLL and Rand.~Pert. score. But the true error token `\textcolor{violet}{left}' is successfully identified by Adv.~Pert.}

        \includegraphics[width=\linewidth]{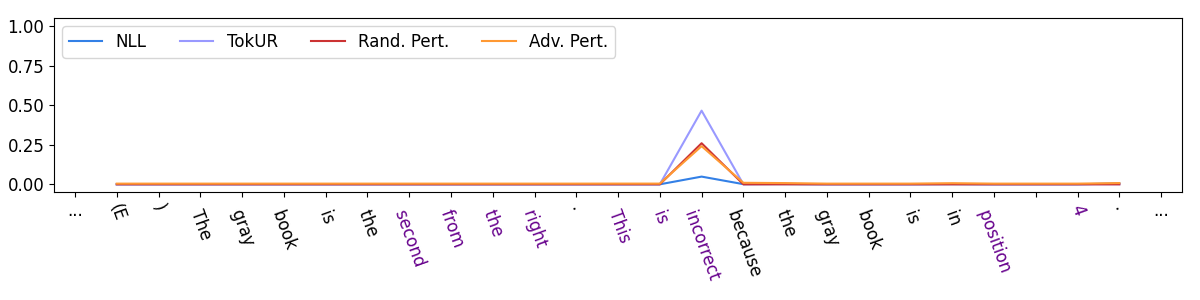}
        {\small \textbf{$\bullet$ Reflection error identified in Qwen.} The LLM is tasked to sort 5 books in different colors. From previous reasoning, it has `the gray book is in \textcolor{violet}{position 4}'. When LLM is identifying the options provided by the question, it failed to deduce that `the gray book is in \textcolor{violet}{position 4}' is equivalent to `the gray book is the \textcolor{violet}{second from the right}', thus it outputs ``This is \textcolor{violet}{incorrect}''. Perturbation based methods successfully identify this token in error.}

        \includegraphics[width=\linewidth]{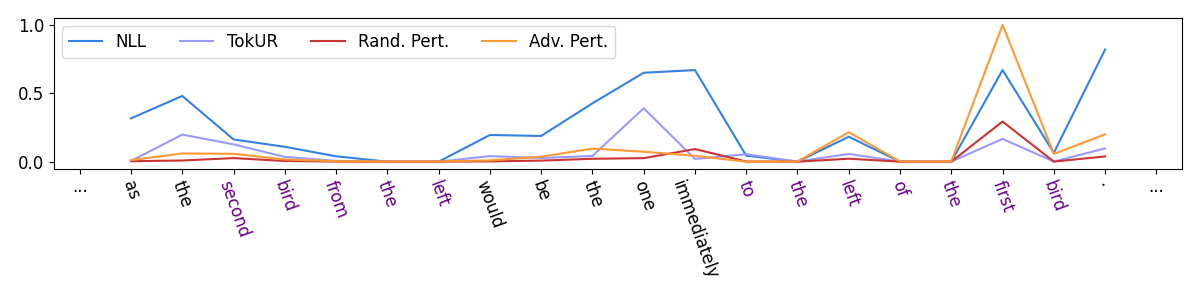}
        {\small \textbf{$\bullet$ Positional inference error identified in Mistral.} The model is asked to sort 5 birds. the error is `the \textcolor{violet}{second bird from the left} would be the one immediately \textcolor{violet}{to the left of the first bird}'. The correct version should be either `to the right of the first bird' or `to the left of the third bird', therefore we observed a small peak at the prior token `left', and a large peak at the latter error token `\textcolor{violet}{first}'.}
    \end{minipage}
    \caption{Extra case studies on logical reasoning task. Each uncertainty score is min-max normalized over the complete output to fit the figure, and only the critical segments are shown.}
    \label{fig:extra_logic_case}
\end{figure}

\begin{figure}[t]
\centering
    \begin{minipage}{0.95\linewidth}
        \includegraphics[width=\linewidth]{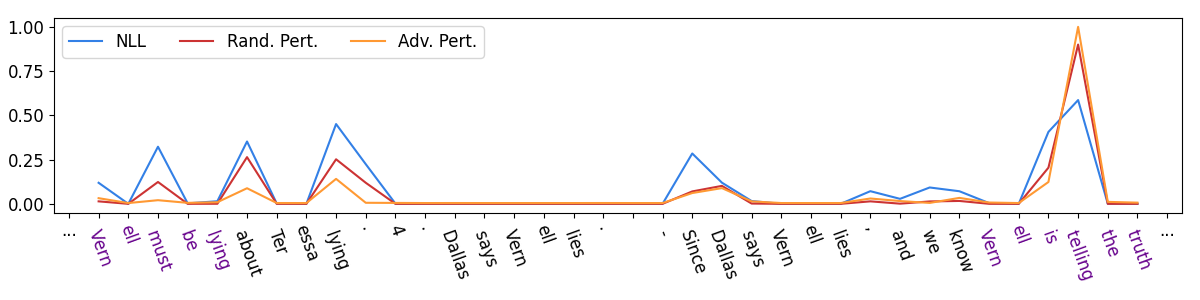}
        {\small \textbf{$\bullet$A successful case in Llama.} The LLM needs to reason who is lying. At first, \textcolor{violet}{`Vernell must be lying about Teressa lying'} is inferred correctly, but later it incorrectly generates \textcolor{violet}{`Vernell is telling the truth'}. Our method successfully identifies the critical error token `telling' which should be `lying' instead.}

        \includegraphics[width=\linewidth]{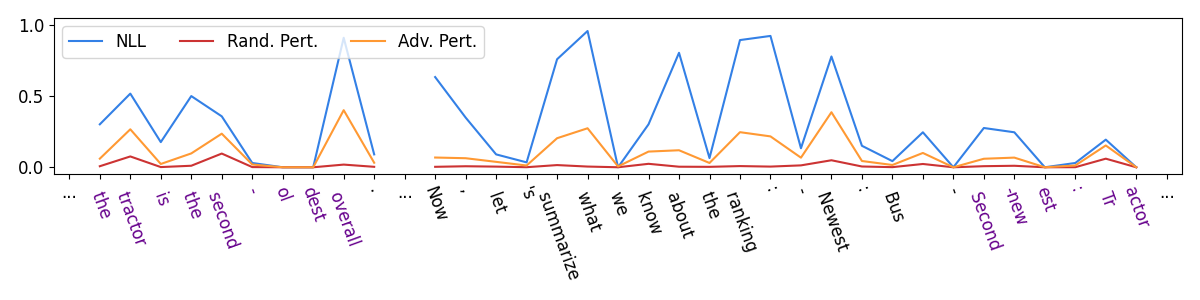}
        {\small \textbf{$\bullet$A failure case in Qwen.} The LLM needs to rank five given vehicles. At first, \textcolor{violet}{`the tractor is the second-oldest'} is inferred correctly by the model, but later it incorrectly concludes \textcolor{violet}{`Second newest: Tractor'}.}
    \end{minipage}
    \caption{Case studies on logical reasoning task. Each uncertainty score is min-max normalized over the complete output to fit the figure, and only part of tokens are shown.}
    \label{fig:main_case}
\end{figure}

\begin{figure}[t]
\centering
    \includegraphics[width=\linewidth]{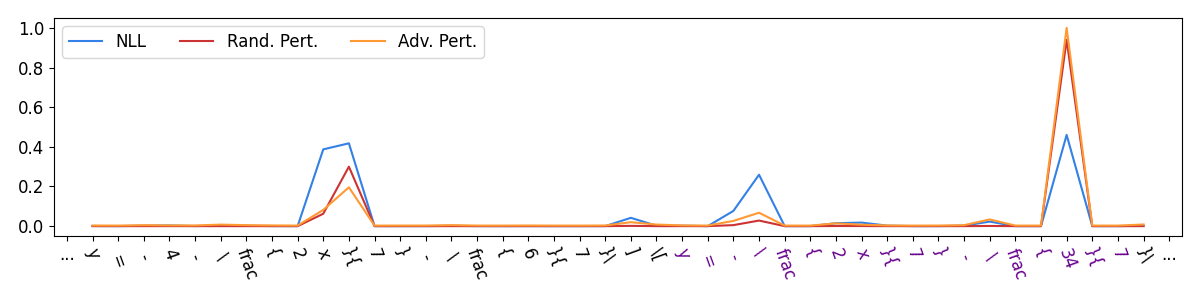}
    \caption{Correct but risky case. The LLM is performing a correct mathematical derivation from $y = -4 - \frac{2x}{7} - \frac{6}{7}$ to \textcolor{violet}{$y = - \frac{2x}{7} - \frac{34}{7}$}. The token \textcolor{violet}{`34'} is flagged as uncertain though correct. Re-generation from this step reveals risk score of 0.6, where 0.5 of the risk is due to ``$y = - \frac{2x}{7} - \frac{32}{7}$'' where model incorrectly generate `32' instead of \textcolor{violet}{`34'}, and 0.1 is because model generate ``$y = - \frac{2x}{7} - \frac{38}{7}$'', both re-generations show the uncertainty of token \textcolor{violet}{`34'}.}
    \label{fig:risk_case}
\end{figure}

\end{document}